\documentclass{article}
\usepackage{iclr2027_conference,times}

\usepackage[utf8]{inputenc}
\usepackage[T1]{fontenc}

\usepackage{url}
\usepackage{booktabs}
\usepackage{amsfonts}
\usepackage{nicefrac}
\usepackage{microtype}
\usepackage{xcolor}
\usepackage{xspace}
\usepackage{graphicx}
\usepackage{wrapfig}

\usepackage{amsmath, amsthm, amssymb}
\usepackage{algorithm,algorithmicx}
\usepackage{algpseudocode}
\usepackage{multirow}
\usepackage{subcaption}
\usepackage{caption}
\usepackage{float}
\usepackage{placeins}

\usepackage{tabularx}
\usepackage[most]{tcolorbox}
\usepackage{hyperref}

\input{lib/dlbook_notation}
\theoremstyle{plain}
\newtheorem{proposition}{Proposition}

\newtheorem{innercustomprop}{Proposition}
\newenvironment{customprop}[1]
{\renewcommand\theinnercustomprop{#1}\innercustomprop}
{\endinnercustomprop}

\newtheorem*{theorem*}{Theorem}
\newtheorem*{proposition*}{Proposition}
\newtheorem*{lemma*}{Lemma}
\newtheorem*{property*}{Property}
\newtheorem*{definition*}{Definition}
\newtheorem*{corollary*}{Corollary}

\theoremstyle{definition}

\newcommand{\methodA}{Loop-aware Activation Scaling\xspace}
\newcommand{\methodB}{Selective Loop-aware Transformation\xspace}
\newcommand{\methodC}{Cross-loop Transition Adapter\xspace}
\newcommand{\methodD}{Trajectory-Aware Calibration\xspace}

\newcommand{\vparagraph}[1]{
	{\noindent \textbf{#1}}
}

\newtcolorbox{takeaway}{
	colback=gray!10,
	colframe=gray!50,
	boxrule=0.5pt,
	arc=3pt,
	left=6pt,
	right=6pt,
	top=6pt,
	bottom=6pt
}

\iclrfinalcopy

\definecolor{academicLink}{HTML}{9F2125}
\definecolor{academicCite}{HTML}{000080}
\definecolor{academicURL}{HTML}{0000FF}
\hypersetup{
  colorlinks=true,
  linkcolor=academicLink,
  citecolor=academicCite,
  urlcolor=academicURL,
  pdfborder={0 0 0},
  breaklinks=true,
  pdftitle={LoopQ: Quantization for Looped Language Models},
  pdfauthor={Rui Fang, Hsi-Wen Chen, Chia-Ming Cheng, Ming-Syan Chen}
}

\title{LoopQ: Quantization for Looped Language Models}
\author{Rui Fang\textsuperscript{1} \quad Hsi-Wen Chen\textsuperscript{1} \quad Chia-Ming Cheng\textsuperscript{2} \quad Ming-Syan Chen\textsuperscript{1} \\
  \textsuperscript{1}National Taiwan University \qquad \textsuperscript{2}MediaTek}
\date{}

\begin{document}

\maketitle
\begingroup
\renewcommand{\thefootnote}{}
\footnotetext{\mbox{\fontsize{7}{9}\selectfont\texttt{\{rfang,hwchen\}@arbor.ee.ntu.edu.tw}, \texttt{cm.cheng@mediatek.com}, \texttt{mschen@ntu.edu.tw}}}
\endgroup
\pagestyle{plain}
\begin{abstract}
Looped language models (LoopLMs) increase effective computational depth by repeatedly applying shared Transformer blocks without increasing the parameter count. However, this reuse makes them particularly sensitive to post-training quantization (PTQ). We systematically study PTQ in LoopLMs and identify three sources of degradation: \textbf{distribution shift across roles}, \textbf{state reuse across loop transitions}, and \textbf{recursive error accumulation}. We propose \textbf{LoopQ}, a loop-aware PTQ framework that preserves a shared quantized backbone while introducing lightweight adaptations. LoopQ combines loop-aware activation scaling and selective transformations to accommodate changing hidden-state distributions, cross-loop state alignment to stabilize transitions, and trajectory-aware calibration to limit accumulated error. Under W4A4 quantization, LoopQ achieves a $67.8\%$ average relative improvement in five-task mean accuracy across four backbones and a conservative estimate of $75.2\%$ for the average relative perplexity reduction across eight backbone--dataset pairs, compared with the strongest static PTQ baseline for each comparison.
\end{abstract}

\section{Introduction}
\label{sec:introduction}

Looped language models (LoopLMs)~\citep{dehghani2018universal,lan2019albert,zhu2025scaling,jeddi2026loopformer,bae2025mixture,fan2024looped} increase effective computational depth by \textbf{recursively reusing shared Transformer blocks} without increasing the parameter count. Their compact parameterization makes them promising for memory- and bandwidth-constrained deployment~\citep{zhou2022transpim,laguna2022hardware,kwon2024lol,yu2024edge}. Post-training quantization (PTQ) could further reduce the storage and data movement required for weights and activations~\citep{alizadeh2020gradient,chmiel2020robust,xiao2023robustmq}. Yet LoopLMs have largely been evaluated in full precision, while existing PTQ methods were developed primarily for conventional LLMs~\citep{xiao2023smoothquant,ashkboos2024quarot,lin2024awq,kim2026turboboa,hu2025ostquant,sun2025flatquant,sanjeet2026mixquant}. Whether these methods remain effective under recursive block reuse is unclear.

Figure~\ref{fig:fig1} illustrates the quantization challenge in LoopLMs. Quantization reduces storage but can sharply degrade accuracy, while hidden-state errors grow across loops and spike near loop transitions. These observations motivate quantization methods that account for recursive state reuse.
The underlying challenge is that the same physical block serves different computational roles across loops and processes changing hidden-state distributions~\citep{zeitoun2026hyperloop}. We distinguish three related sources of error. \textbf{(1) Distribution Shift across Roles:} shared quantization scales and transformations may be poorly suited to the activation magnitudes and geometry of individual loops. \textbf{(2) State Reuse across Transitions:} quantized outputs feed subsequent loops, carrying state mismatch across boundaries. \textbf{(3) Recursive Error Accumulation:} errors introduced in earlier loops propagate through subsequent computations and can be amplified. Section~\ref{sec:error} analyzes how shared quantization incurs mismatch and how loop-wise errors propagate along the recurrent trajectory.

Making every quantization transformation loop-specific could reduce such mismatch, but would increase storage and weaken the parameter-sharing advantage of LoopLMs. We therefore propose \textbf{LoopQ}, which retains a shared quantized backbone and introduces loop-dependent adaptation where it is most useful. As shown in the right panel of Figure~\ref{fig:fig1}, \textbf{\methodA} adjusts activation scales across loops to address magnitude drift. \textbf{\methodB} uses a \textit{sharing-sensitivity score} to select transformation groups that benefit from loop-specific treatment, addressing residual geometry mismatch while limiting storage overhead. \textbf{\methodC} corrects hidden states at loop transitions with a lightweight adapter. Finally, \textbf{\methodD} jointly calibrates these components along the recurrent trajectory, balancing alignment with full-precision boundary states against agreement in final predictions.

\textbf{Our contributions are threefold.} \textbf{(i)} We analyze quantization under recursive block reuse, characterizing the cost of sharing a quantization configuration across loops and the propagation of state error through subsequent loops. \textbf{(ii)} We develop LoopQ, which combines loop-dependent scaling and selective transformations with boundary-state correction and trajectory-aware calibration while retaining a shared quantized backbone. \textbf{(iii)} LoopQ achieves a $67.8\%$ average relative gain in five-task mean accuracy, a conservatively estimated average relative perplexity reduction of $75.2\%$, and up to $2.13\times$ BF16 decode throughput on long prompts.

\begin{figure}[t]
	\centering
	\includegraphics[width=\textwidth]{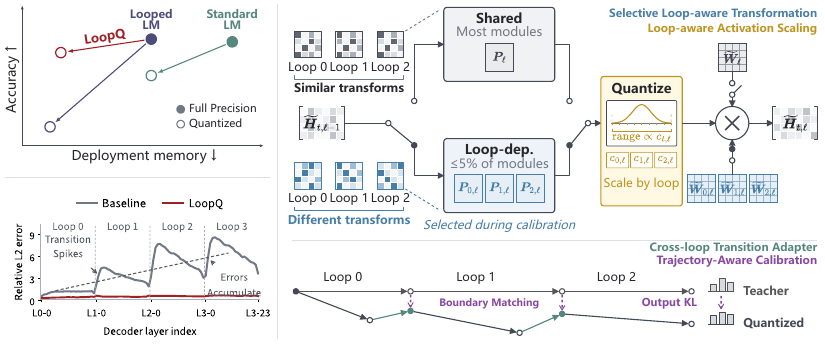}
	\caption{\textbf{Left:} A schematic accuracy--storage comparison of looped and standard, non-recursive LMs before and after quantization (top), and relative $L_2$ error of Ouro 1.4B across four loops for the PTQ baseline and LoopQ (bottom). \textbf{Right:} LoopQ uses loop-aware activation scales and restricts loop-dependent transformations to selected modules, with the transforms folded into stored weights (upper). The CTA adjusts hidden states at loop boundaries, while TAC aligns post-adapter states with full-precision boundary states and final predictions through output KL (lower).
		\label{fig:fig1}}
\end{figure}

\section{Problem Formulation}
\label{sec:problem}

We first define the recursive computation in a LoopLM, then describe PTQ with quantization parameters shared across loops.

\vparagraph{LoopLM.}
A LoopLM applies a shared stack of $L$ Transformer layers for $T$ loops. Let $\vx$ be an input sequence of $n$ tokens, and let $f_\ell(\cdot;\vtheta_\ell)$ denote layer $\ell$. The hidden state at layer $\ell$ of loop $t$ is $\mH_{t,\ell}\in\R^{n\times d}$, where $d$ is the hidden dimension. We consider one input sequence and omit the batch dimension. Starting from $\mH_{0,0}=\mathrm{Embed}(\vx)$, the model computes
\begingroup\small
\begin{equation}
    \label{eq:fp_recurrent}
    \mH_{t,\ell}=f_\ell(\mH_{t,\ell-1};\vtheta_\ell), \qquad
    \mH_{t+1,0}=\mH_{t,L},
\end{equation}
\endgroup
for $t\in\{0,\ldots,T-1\}$ and $\ell\in\{1,\ldots,L\}$. Writing $\mH_t:=\mH_{t,0}$ and $F:=f_L\circ\cdots\circ f_1$, we obtain the recurrence $\mH_{t+1}=F(\mH_t)$. After $T$ loops, the output logits are $\mZ=\mathrm{Proj}(\mH_T)\in\R^{n\times v}$, where $v$ is the vocabulary size.

For the quantization analysis, we represent a layer by a linear projection with weight matrix $\mW_\ell$, so that $\mH_{t,\ell}=\mH_{t,\ell-1}\mW_\ell^\top$. This abstraction isolates the effect of reusing a quantized mapping; the LoopLM itself consists of full Transformer layers.

\vparagraph{Post-training Quantization on LoopLMs.}
We begin with a conventional PTQ configuration that applies the same quantization parameters to each shared layer on every loop. Direct low-bit quantization can incur substantial error when weights or activations have heavy tails or large differences in scale across channels. Transformation-based PTQ addresses this by applying an invertible transformation $\mP_\ell$ to the activations and the corresponding inverse transformation to the weights. The full-precision mapping is unchanged:
\begingroup\small
\begin{equation}
    \mH_{t,\ell}
    =
    (\mH_{t,\ell-1}\mP_\ell)
    (\mW_\ell\mP_\ell^{-\top})^\top .
\end{equation}
\endgroup
With a shared transformation $\mP_\ell$ and activation quantization step size $c_\ell$, quantized inference takes the form
\begingroup\small
\begin{equation}
    \label{eq:quant}
    \widetilde{\mH}_{t,\ell}
    =
    Q_a(\widetilde{\mH}_{t,\ell-1}\mP_\ell; c_\ell)\,
    Q_w(\mW_\ell\mP_\ell^{-\top})^\top,
    \qquad
    \widetilde{\mH}_{t+1,0}=\widetilde{\mH}_{t,L}.
\end{equation}
\endgroup
Here, $Q_w$ and $Q_a$ denote weight and activation quantization. For a scalar activation $x$, symmetric uniform quantization is
\begingroup\small
\begin{equation}
    \label{eq:quant_a}
    Q_a(x;c_\ell)
    =
    c_\ell\cdot\mathrm{clip}\!\left(
        \left\lfloor x/c_\ell\right\rceil,
        q_{\min},q_{\max}
    \right),
\end{equation}
\endgroup
where $c_\ell>0$ sets the quantization step size, $\lfloor\cdot\rceil$ denotes rounding to the nearest integer, and $q_{\min}$ and $q_{\max}$ are the representable integer bounds. We suppress activation-group indices for clarity.

Equation~\eqref{eq:quant} exposes two coupled difficulties specific to recursive reuse. First, each layer uses the same $\mP_\ell$ and $c_\ell$ across loops even though its activation distribution changes. Second, the quantized output of one loop becomes the input to the next, allowing local quantization errors to alter subsequent states and accumulate along the trajectory. We analyze these effects in the next section.

\section{Where Quantization Error Arises in LoopLMs}
\label{sec:error}

We analyze two coupled effects of quantizing a recursively reused Transformer stack: mismatch between shared quantization parameters and loop-dependent activations, and propagation of error through reused hidden states. Proofs are provided in Appendix~\ref{apx:proof}.

\vparagraph{Distribution Shift across Roles.}
At a given layer, all loops share the same quantization configuration, but the layer receives different hidden-state distributions as recursion proceeds. A shared scale $c_\ell$ must balance resolution for small activations against clipping of large ones. Changes in dominant directions and channel-wise energy can also make different loops favor different transformations, leaving a shared $\mP_\ell$ suboptimal for some loops.

\begin{proposition}
    \label{prop:inter_layer}
    For $\vx\sim\train$, let $\mH_{t,\ell-1}(\vx)$ be the hidden state on the fixed full-precision trajectory, and define $\mX_{t,\ell}(\vx;\mP):=\mH_{t,\ell-1}(\vx)\mP$. With $Q_a$ from Eq.~\eqref{eq:quant_a} applied elementwise, the expected transformed-activation quantization error is
\begingroup\small
\begin{equation*}
    \varepsilon_{t,\ell}(\mP,c)
    :=\E_{\vx\sim\train}\!\left[
        \left\|Q_a(\mX_{t,\ell}(\vx;\mP);c)
        -\mX_{t,\ell}(\vx;\mP)\right\|_F^2
    \right].
\end{equation*}
\endgroup
Assume that all minima below are finite and attained over the same admissible set of invertible $\mP$ and $c>0$. Let $\varepsilon_{t,\ell}^{*}:=\min_{\mP,c}\varepsilon_{t,\ell}(\mP,c)$ and define the cost of sharing one configuration across loops as
\begingroup\small
\begin{equation*}
    \Delta_\ell^{\mathrm{share}}
    :=\min_{\mP_\ell,c_\ell}
        \sum_{t=0}^{T-1}\varepsilon_{t,\ell}(\mP_\ell,c_\ell)
      -\sum_{t=0}^{T-1}\varepsilon_{t,\ell}^{*}.
\end{equation*}
\endgroup
Then $\Delta_\ell^{\mathrm{share}}\ge0$, with equality if and only if a single configuration is optimal for every loop. In particular, if any two loops have disjoint sets of optimal configurations, then $\Delta_\ell^{\mathrm{share}}>0$: every shared configuration incurs excess transformed-activation quantization error in at least one loop.

\end{proposition}

Proposition~\ref{prop:inter_layer} compares the best shared configuration with loop-specific optima. The excess activation error vanishes when all loops admit a common optimum and is positive when their optima are incompatible. The accompanying analysis gives a sufficient condition under which magnitude drift requires different scales, and a local quadratic estimate of the cost of sharing transformations preferred by different loops. This distinction motivates loop-dependent activation scaling and selective relaxation of transformation sharing (Sec.~\ref{sec:loopq_slt}).

\vparagraph{Error Propagation across Loops.}
Quantization also changes the state passed between loops. Recall the full-precision recurrence $\mH_{t+1}=F(\mH_t)$ and its quantized counterpart $\widetilde{\mH}_{t+1}=\widetilde F_t(\widetilde{\mH}_t)$. Their discrepancy reflects both error in the quantized mapping at the current input and the difference between $\widetilde{\mH}_t$ and $\mH_t$. Earlier errors can therefore propagate through state reuse, consistent with the loop-boundary behavior in Figure~\ref{fig:fig1}.

\begin{proposition}
    \label{prop:inter_loop}
    Define the loop error as $\varepsilon_t:=\|\widetilde{\mH}_t-\mH_t\|_F$. Assume that both recurrences start from the same initial state and, for each loop $t$, the full-precision function $F$ is $\gamma_t$-Lipschitz on the segment between $\widetilde{\mH}_t$ and $\mH_t$. Then $\varepsilon_{t+1}\le\varepsilon_t^{\mathrm{quant}}+\gamma_t\varepsilon_t$, where $\varepsilon_t^{\mathrm{quant}}:=\|\widetilde F_t(\widetilde{\mH}_t)-F(\widetilde{\mH}_t)\|_F$. Unrolling this recursion yields
\begingroup\small\begin{equation*}
	\varepsilon_T\le\sum_{\tau=0}^{T-1}\left(\prod_{t=\tau+1}^{T-1}\gamma_t\right)\varepsilon_\tau^{\mathrm{quant}}.
\end{equation*}\endgroup
Thus, each loop-wise quantization error is propagated through downstream loop sensitivities, with earlier errors passing through a longer chain of transformations.

\end{proposition}

Proposition~\ref{prop:inter_loop} bounds the final state error by loop-wise mapping errors weighted by downstream sensitivities. Errors introduced earlier pass through more loops and may have a larger effect, motivating correcting states before reuse and calibrating the quantized model along the full recurrent trajectory.

\section{LoopQ: Loop-Aware PTQ for Looped Transformers}
\label{sec:loopq}

LoopQ builds on a shared, transformation-based quantized backbone~\citep{sun2025flatquant} and introduces four coordinated adaptations. \textbf{(i)} \textit{Loop-aware Activation Scaling} (LAS) uses loop-dependent activation scales to accommodate changes in magnitude. \textbf{(ii)} \textit{Selective Loop-aware Transformation} (SLT) uses a sharing-sensitivity score to identify transformation groups that benefit most from loop-dependent configurations, while keeping the remaining groups shared. \textbf{(iii)} \textit{Cross-loop Transition Adapter} (CTA) adjusts the quantized hidden state before it enters the next loop, limiting the state mismatch carried forward. \textbf{(iv)} \textit{Trajectory-Aware Calibration} (TAC) jointly optimizes these components along the recurrent trajectory, aligning post-adapter boundary states and final predictions with the full-precision model. Implementation details and calibration procedures are provided in Appendix~\ref{apx:detailed_implementation}.

\subsection{\methodA}

Sharing quantization parameters across loops creates a trade-off when the same layer receives activations with different magnitudes. The scale-drift analysis accompanying Proposition~\ref{prop:inter_layer} shows that, under such variation, a single activation step size may be suboptimal even if the shared transformation $\mP_\ell$ is fixed. We first address this source of mismatch through loop-aware activation scaling.

The activation scale $c_\ell$ controls both quantization resolution and dynamic range. Because a shared module processes loop-dependent hidden states, a scale large enough to accommodate the activations in one loop may be too coarse for smaller activations in another. Conversely, a scale tuned to preserve resolution for small activations may clip larger ones~\citep{xiao2023smoothquant,sun2025flatquant}. The former increases rounding error, while the latter discards information that subsequent loops cannot recover. Since each quantized loop output becomes the input to the next, either type of error may propagate along the recurrent trajectory. LoopQ alleviates this trade-off by replacing the shared scale $c_\ell$ with loop-dependent scales $c_{t,\ell}$, yielding $Q_a(x;c_{t,\ell})$.

The scales are calibrated offline and selected according to the physical loop at inference time. Each activation quantizer stores one scale per loop and feature group, a small addition relative to the shared backbone. Loop-dependent scaling adapts activation magnitudes under the shared transformation $\mP_\ell$. The selective mechanism developed next addresses differences in the transformations preferred by individual loops while limiting the storage required for loop-dependent weight representations.

\subsection{\methodB}
\label{sec:loopq_slt}

Loop-dependent scaling reduces magnitude mismatch while keeping $\mP_\ell$ shared. Proposition~\ref{prop:inter_layer} shows that sharing increases activation error when loop-specific optima conflict; a fixed-scale quadratic analysis quantifies the cost of tying their preferred transformations. This motivates separate transformations where activation geometries differ and sharing elsewhere to limit storage.

SLT ranks candidate transformation groups using a \textit{sharing-sensitivity score} based on the recurrent calibration loss $\Ls$ (Sec.~\ref{sec:loopq_unified}). Ideally, this score would measure the reduction in $\Ls$ obtainable by allowing loop-dependent transforms $\mP_{t,\ell}$ while holding the remaining configuration fixed. Unlike the fixed-trajectory activation error in Proposition~\ref{prop:inter_layer}, this criterion evaluates the effect of relaxing sharing under the calibration objective. A larger score indicates that a shared transform is more costly for that group.\footnote{Selection is performed over grouped transformation sites rather than entire layers.}

Evaluating this loss reduction exactly would require separately optimizing each candidate. We approximate its local effect on $\Ls$ using curvature. Writing $\mP_{t,\ell}=\mP_\ell+\mDelta_{t,\ell}$ as a local perturbation of the shared configuration, a second-order expansion for one transform occurrence gives
\begingroup\small
\begin{equation}
    \label{eq:slt_taylor}
    \Ls(\mP_{t,\ell})
    \approx \Ls(\mP_\ell)
    +\left\langle\nabla_{\mP_{t,\ell}}\Ls,\mDelta_{t,\ell}\right\rangle_F
    +\frac12\left\langle\mDelta_{t,\ell},
    \nabla^2_{\mP_{t,\ell}}\Ls[\mDelta_{t,\ell}]\right\rangle_F.
\end{equation}
\endgroup
To estimate the cost of constraining these updates to be shared across loops, we use a blockwise generalized Gauss--Newton (GGN) approximation~\citep{martens2015optimizing} based on predictive KL:
\begingroup\small
\begin{equation}
    \label{eq:slt_ggn_block}
    \mG_{\ell j}
    =\frac1N\sum_{i=1}^{N}
    \mJ_{i,\ell j}^{\top}\mF_i\mJ_{i,\ell j}.
\end{equation}
\endgroup
Here, $\mF_i=\diag(\vp_i)-\vp_i\vp_i^\top$, with $\vp_i=\mathrm{softmax}(\widetilde{\mZ}_i)$, is the Fisher information of the quantized predictive distribution at token $i$. The Jacobian $\mJ_{i,\ell j}$ relates its logits to rotation coordinate $j$ of transformation group $\ell$ across all loops, and $N$ counts valid prediction tokens. Each $\mG_{\ell j}$ is a $T\times T$ block that captures coupling among loop occurrences of one coordinate while omitting cross-coordinate coupling. With positive damping, these blocks yield the sharing-sensitivity score $S_\ell$, which estimates the local loss increase from projecting trial loop-dependent coordinates onto a shared configuration. Appendix~\ref{apx:slt_optimization} details the curvature estimation, score computation, reversible selection procedure, and calibration schedule.

LoopQ uses these scores to select loop-dependent transformation groups during joint calibration. It first permits loop-dependent updates at all candidate sites, then progressively reduces the adaptation budget. At each stage, groups outside the budget are projected to shared coordinates but remain eligible for later selection. This reversible search allows the selected groups to change as other mismatches are corrected. After selection, the retained groups use loop-dependent transforms $\mP_{t,\ell}$, the remaining groups use shared transforms $\mP_\ell$, and calibration continues with this structure fixed.

\subsection{\methodC}
\label{sec:loopq_cta}

After loop-dependent scaling and selective transformation reduce mismatch within shared layers, the resulting quantized state is still passed directly to the next loop. Proposition~\ref{prop:inter_loop} shows how an existing state error can propagate through subsequent loop computations. LoopQ therefore introduces a lightweight Cross-loop Transition Adapter (CTA) to adjust the state at each loop boundary without modifying the shared quantized backbone.

For $t<T-1$, the adapter takes the quantized loop input and output and produces the next loop's input, $\widetilde{\mH}_{t+1,0}=A_t(\widetilde{\mH}_{t,0},\widetilde{\mH}_{t,L})$, where
\begingroup\small
\begin{equation}
    \label{eq:cta}
    A_t(\widetilde{\mH}_{t,0},\widetilde{\mH}_{t,L})
    =\widetilde{\mH}_{t,L}
    +g_t(\widetilde{\mH}_{t,0},\Delta\widetilde{\mH}_t)
        \odot\Delta\widetilde{\mH}_t
    +\Pi_{\Delta\widetilde{\mH}_t}^{\perp}\!
        \left[\mR_t(\widetilde{\mH}_{t,0},\Delta\widetilde{\mH}_t)\right].
\end{equation}
\endgroup
Here, $\Delta\widetilde{\mH}_t=\widetilde{\mH}_{t,L}-\widetilde{\mH}_{t,0}$ is the update produced by loop $t$. The gate $g_t$ modulates this update, while $\mR_t$ uses a low-rank bottleneck to generate an additional correction from the loop input and update. The projection $\Pi_{\Delta\widetilde{\mH}_t}^{\perp}$ removes the component of that correction parallel to the update, applied independently to each token and feature group. Thus, the two branches adjust the existing update and supply a complementary directional correction before the state is reused. Their network parameters are shared across loop boundaries, with lightweight boundary-specific offsets providing adaptation. Implementation details are provided in Appendix~\ref{apx:cta_details}.

\subsection{\methodD}
\label{sec:loopq_unified}

The effects of LAS, SLT, and CTA extend beyond the loop in which each component acts: changing an intermediate state also changes the inputs to subsequent loops. Trajectory-Aware Calibration (TAC) therefore optimizes these components jointly through the full recurrent computation, using the full-precision model as a teacher~\citep{hinton2015distilling}.

For calibration inputs $\vx\sim\train$, LoopQ minimizes
\begingroup\small
\begin{equation}
    \label{eq:loopq_unified_loss}
    \Ls
    =\E_{\vx\sim\train}\!\left[
        \KL\!\left(\mZ\,\middle\|\,\widetilde\mZ\right)
        +\sum_{t=0}^{T-2}\mu_t
        \left\|
            \widetilde\mH_{t+1,0}-\mH_{t+1,0}
        \right\|_F^{2}
    \right],
\end{equation}
\endgroup
where $\KL(\mZ\|\widetilde\mZ)$ denotes the KL divergence between the predictive distributions induced by the full-precision and quantized logits. The state $\widetilde\mH_{t+1,0}=A_t(\widetilde\mH_{t,0},\widetilde\mH_{t,L})$ is the post-adapter input to loop $t+1$ from Eq.~\eqref{eq:cta}.

The predictive KL aligns the final outputs, while the boundary terms match post-adapter states to their full-precision counterparts before those states are reused. The effective coefficient $\mu_t$ summarizes token-wise adaptive weighting based on relative boundary-state error and terminal predictive KL. When boundary error is large relative to prediction error, this weighting reduces the influence of state matching so that intermediate-state discrepancies do not dominate calibration.

We optimize Eq.~\eqref{eq:loopq_unified_loss} through the unrolled recurrence, allowing losses at later loops to update shared transforms $\mP_\ell$ and earlier loop-dependent components. The latter include activation scales $c_{t,\ell}$, the transforms $\mP_{t,\ell}$ retained by SLT, and transition adapters $A_t$; the pretrained backbone weights remain frozen. Joint calibration thus accounts for both the boundary states passed between loops and their effects on the final prediction. Appendix~\ref{apx:detailed_implementation} details the full calibration procedure.

\section{Experiments}
\label{sec:experiment}

\subsection{Experimental Setup}
\label{sec:setup}

\vparagraph{Benchmarks and Baselines.}
Following prior quantization work, we evaluate HellaSwag
(\texttt{HS})~\citep{zellers2019hellaswag},
WinoGrande (\texttt{WG})~\citep{sakaguchi2021winogrande},
ARC-Challenge (\texttt{AC})~\citep{clark2018think},
MMLU (\texttt{MM})~\citep{hendrycks2020measuring},
and OpenAI LAMBADA (\texttt{LB})~\citep{paperno2016lambada,radford2019language}.
We use \texttt{acc\_norm} where available and report the mean accuracy
across these five tasks. We also measure perplexity on WikiText-2
(\texttt{WT2})~\citep{merity2016pointer} and LAMBADA (\texttt{OP}).
The PTQ baselines are SmoothQuant~\citep{xiao2023smoothquant},
QuaRot~\citep{ashkboos2024quarot},
SpinQuant~\citep{liu2025spinquant},
FlatQuant~\citep{sun2025flatquant},
and MixQuant~\citep{sanjeet2026mixquant}.
BF16 is the full-precision reference, while symmetric round-to-nearest
(RTN) quantization provides a direct-quantization baseline.

\vparagraph{Backbones and Implementation.} We evaluate four backbones from three LoopLM architectures:
\textit{Ouro}~\citep{zhu2025scaling}, with 24 shared layers and
4 loops in its 1.4B variant and 48 shared layers and 4 loops in its
2.6B variant; shortcut-modulated
\textit{LoopFormer}~\citep{jeddi2026loopformer}, with 3 shared layers
and 8 loops; and partially looped
\textit{Parcae}~\citep{prairie2026parcae}, with 4 shared layers
repeated over 8 loops and 8 non-looped layers.
The main W4A4 and W4A8 comparisons use static activation quantization,
symmetric uniform quantizers, and round-to-nearest discretization.
All static PTQ baselines collect activation statistics across the
complete recursive forward pass. LoopQ uses 1,024 calibration examples
and FlatQuant-style Kronecker-decomposed
transforms~\citep{sun2025flatquant}.
Implementation and calibration details appear in
Appendix~\ref{apx:implementation_details}.

\subsection{Quantization Performance Across Tasks and Backbones}
\label{sec:quantization_quality}

\begin{table}[t]
\centering\small
\caption{Main results under W4A4 and W4A8. PPL values above $10^3$ are displayed as $>10^3$.}
\label{tab:main_results}
\setlength{\tabcolsep}{3pt}
\resizebox{\textwidth}{!}{\begin{tabular}{cc ccccccc ccccccc}
\toprule
&                 & \multicolumn{7}{c}{\textbf{W4A4}}             & \multicolumn{7}{c}{\textbf{W4A8}}                                                                                                                                                                                                                                                                                                           \\
\cmidrule(lr){3-9} \cmidrule(lr){10-16}
&                 & \multicolumn{5}{c}{\textbf{Tasks}~$\uparrow$} & \multicolumn{2}{c}{\textbf{Perplexity}~$\downarrow$} & \multicolumn{5}{c}{\textbf{Tasks}~$\uparrow$} & \multicolumn{2}{c}{\textbf{Perplexity}~$\downarrow$}                                                                                                                                                                                 \\
\cmidrule(lr){3-7} \cmidrule(lr){8-9} \cmidrule(lr){10-14} \cmidrule(lr){15-16}
\textbf{Model}                                                            & \textbf{Method} & \textbf{HS}                                   & \textbf{WG}                                          & \textbf{LB}                                   & \textbf{AC}                                          & \textbf{MM}     & \textbf{WT2}   & \textbf{OP}    & \textbf{HS}     & \textbf{WG}     & \textbf{LB}     & \textbf{AC}     & \textbf{MM}     & \textbf{WT2}   & \textbf{OP}    \\
\midrule
\multirow{8}{*}{\rotatebox[origin=c]{90}{\textbf{Ouro 1.4B}}}             & \textit{BF16}   & \textit{0.7055}                               & \textit{0.6606}                                      & \textit{0.6146}                               & \textit{0.5341}                                      & \textit{0.6886} & \textit{11.18} & \textit{6.22}  & \textit{0.7055} & \textit{0.6606} & \textit{0.6146} & \textit{0.5341} & \textit{0.6886} & \textit{11.18} & \textit{6.22}  \\
& RTN       & 0.2777                                        & 0.4941                                               & 0.0033                                        & 0.2389                                               & 0.2282          & $>10^{3}$      & $>10^{3}$      & 0.5022          & 0.5866          & 0.3220          & 0.3745          & 0.3569          & 43.89          & 46.72          \\
& SmoothQuant     & 0.2575                                        & 0.5170                                               & 0.0000                                        & 0.2457                                               & 0.2499          & $>10^{3}$      & $>10^{3}$      & 0.6010          & 0.5722          & 0.2851          & 0.4104          & 0.2910          & 33.27          & 60.21          \\
& QuaRot          & 0.2908                                        & 0.4972                                               & 0.0293                                        & 0.2270                                               & 0.2366          & 451.44         & $>10^{3}$      & 0.6817          & 0.6195          & 0.5523          & 0.4983          & 0.5577          & 15.27          & 8.95           \\
& SpinQuant       & 0.2914                                        & 0.4925                                               & 0.0078                                        & 0.2210                                               & 0.2352          & 653.78         & $>10^{3}$      & 0.4838          & 0.5630          & 0.3647          & 0.3745          & 0.3240          & 27.88          & 27.61          \\
& FlatQuant       & 0.2698 & 0.4483 & 0.0000 & 0.2278 & 0.2358 & $>10^{3}$ & $>10^{3}$ & 0.6948          & \textbf{0.6511} & 0.5962          & 0.4983          & \textbf{0.6336} & 13.53          & 6.92           \\
& MixQuant & 0.2719 & 0.4665 & 0.0006 & 0.2329 & 0.2409 & $>10^{3}$ & $>10^{3}$ & 0.6769 & 0.6125 & 0.4956 & 0.5068 & 0.4464 & 13.94 & 12.67 \\
& LoopQ           & \textbf{0.6689}                               & \textbf{0.6125}                                      & \textbf{0.5812}                               & \textbf{0.4898}                                      & \textbf{0.6022} & \textbf{13.22} & \textbf{8.20}  & \textbf{0.6954} & 0.6456          & \textbf{0.6255} & \textbf{0.5145} & 0.6206          & \textbf{13.29} & \textbf{6.09}  \\
\midrule
\multirow{8}{*}{\rotatebox[origin=c]{90}{\textbf{Ouro 2.6B}}}             & \textit{BF16}   & \textit{0.7636}                               & \textit{0.6867}                                      & \textit{0.7013}                               & \textit{0.5802}                                      & \textit{0.7380} & \textit{10.51} & \textit{4.22}  & \textit{0.7636} & \textit{0.6867} & \textit{0.7013} & \textit{0.5802} & \textit{0.7380} & \textit{10.51} & \textit{4.22}  \\
& RTN       & 0.2811                                        & 0.5043                                               & 0.0361                                        & 0.2244                                               & 0.2311          & 962.74         & $>10^{3}$      & 0.5517          & 0.5777          & 0.4093          & 0.3729          & 0.3386          & 17.21          & 22.07          \\
& SmoothQuant     & 0.2589                                        & 0.5012                                               & 0.0000                                        & 0.2662                                               & 0.2455          & $>10^{3}$      & $>10^{3}$      & 0.3219          & 0.5099          & 0.0693          & 0.2363          & 0.2297          & 198.38         & $>10^{3}$      \\
& QuaRot          & 0.2967                                        & 0.4964                                               & 0.0782                                        & 0.2329                                               & 0.2307          & 301.26         & $>10^{3}$      & 0.5389          & 0.5659          & 0.4157          & 0.3874          & 0.4419          & 15.78          & 21.53          \\
& SpinQuant       & 0.2805                                        & 0.4862                                               & 0.0200                                        & 0.2193                                               & 0.2316          & 621.14         & $>10^{3}$      & 0.5651          & 0.5754          & 0.4611          & 0.4044          & 0.4432          & 15.81          & 17.26          \\
& FlatQuant       & 0.2560                                        & 0.4917                                               & 0.0000                                        & 0.2560                                               & 0.2455          & $>10^{3}$      & $>10^{3}$      & 0.6393          & 0.6085          & 0.3895          & 0.4309          & 0.4519          & 25.51          & 26.07          \\
& MixQuant        & 0.2900                                        & 0.4862                                               & 0.0068                                        & 0.2406                                               & 0.2347          & 853.32         & $>10^{3}$      & 0.7447          & 0.6740          & 0.6272          & 0.5503          & \textbf{0.7145} & 13.92          & 6.27           \\
& LoopQ           & \textbf{0.6860}                               & \textbf{0.6227}                                      & \textbf{0.5680}                               & \textbf{0.5034}                                      & \textbf{0.5265} & \textbf{15.89} & \textbf{8.17}  & \textbf{0.7555} & \textbf{0.6780} & \textbf{0.6837} & \textbf{0.5819} & 0.7007          & \textbf{11.38} & \textbf{4.47}  \\
\midrule
\multirow{8}{*}{\rotatebox[origin=c]{90}{\textbf{LoopFormer 3$\times$8}}} & \textit{BF16}   & \textit{0.3218}                               & \textit{0.5107}                                      & \textit{0.3776}                               & \textit{0.2457}                                      & \textit{0.2306} & \textit{35.62} & \textit{26.26} & \textit{0.3218} & \textit{0.5107} & \textit{0.3776} & \textit{0.2457} & \textit{0.2306} & \textit{35.62} & \textit{26.26} \\
& RTN       & 0.2731                                        & 0.4988                                               & 0.0016                                        & \textbf{0.2457}                                      & 0.2391          & $>10^{3}$      & $>10^{3}$      & 0.3156          & 0.5099          & 0.3621          & \textbf{0.2389} & \textbf{0.2311} & 39.00          & 29.51          \\
& SmoothQuant     & 0.2653                                        & 0.4972                                               & 0.0004                                        & 0.2440                                               & \textbf{0.2418} & $>10^{3}$      & $>10^{3}$      & 0.3019          & 0.5178          & 0.1296          & 0.2363          & 0.2301          & 113.70         & 396.20         \\
& QuaRot          & 0.3087                                        & 0.5264                                               & 0.0869                                        & 0.2389                                               & 0.2295          & 56.05          & 290.49         & \textbf{0.3207} & 0.5209          & 0.3705          & 0.2363          & 0.2303          & 37.47          & 29.87          \\
& SpinQuant       & 0.2959                                        & 0.5051                                               & 0.2331                                        & 0.2176                                               & 0.2301          & 98.20          & 88.91          & 0.3049          & 0.5146          & 0.2451          & 0.2244          & 0.2291          & 42.83          & 92.20          \\
& FlatQuant       & 0.2782                                        & 0.5170                                               & 0.0276                                        & 0.2184                                               & 0.2318          & 462.39         & $>10^{3}$      & 0.3186          & 0.5193          & 0.3768 & 0.2338          & 0.2300          & 38.25          & \textbf{26.85} \\
& MixQuant        & 0.3092                                        & \textbf{0.5406}                                      & 0.2822                                        & 0.2261                                               & 0.2304          & 39.73          & 54.82          & 0.3188          & \textbf{0.5280} & 0.3578          & 0.2346          & 0.2302          & 41.87 & 30.79          \\
& LoopQ           & \textbf{0.3176}                               & 0.5043                                               & \textbf{0.3466}                               & 0.2372                                               & 0.2301          & \textbf{37.36} & \textbf{34.46} & 0.3174          & 0.5217          & \textbf{0.3788}          & 0.2381          & 0.2301          & \textbf{37.44}          & 27.62          \\
\midrule
\multirow{8}{*}{\rotatebox[origin=c]{90}{\textbf{Parcae 370M}}}           & \textit{BF16}   & \textit{0.4409}                               & \textit{0.5446}                                      & \textit{0.3565}                               & \textit{0.2910}                                      & \textit{0.2558} & \textit{26.98} & \textit{32.55} & \textit{0.4409} & \textit{0.5446} & \textit{0.3565} & \textit{0.2910} & \textit{0.2558} & \textit{26.98} & \textit{32.55} \\
& RTN       & 0.2992                                        & 0.5091                                               & 0.0097                                        & 0.2227                                               & 0.2455          & 797.53         & $>10^{3}$      & 0.3983          & 0.5130          & 0.1013          & 0.2688          & 0.2447          & 82.00          & 645.47         \\
& SmoothQuant     & 0.3252                                        & 0.4980                                               & 0.0833                                        & 0.2270                                               & 0.2379          & 489.94         & 914.59         & 0.3511          & 0.5162          & 0.1310          & 0.2312          & 0.2355          & 162.19         & 250.89         \\
& QuaRot          & 0.3534                                        & 0.5146                                               & 0.1289                                        & 0.2321                                               & 0.2349          & 273.18         & 454.18         & 0.4214          & 0.5272          & 0.3266          & 0.2756          & 0.2334          & 34.41          & 45.35          \\
& SpinQuant       & 0.3057                                        & \textbf{0.5185}                                      & 0.0305                                        & 0.2201                                               & 0.2354          & 837.26         & $>10^{3}$      & 0.3765          & 0.5051          & 0.1807          & 0.2568          & 0.2481          & 207.42         & 276.29         \\
& FlatQuant       & 0.3191                                        & 0.4972                                               & 0.0695                                        & 0.2278                                               & 0.2310          & 219.36         & $>10^{3}$      & 0.4319          & \textbf{0.5422} & 0.3278          & 0.2867          & 0.2366          & 28.80          & 40.05          \\
& MixQuant        & 0.3261                                        & 0.5162                                               & 0.0534                                        & 0.2526                                               & 0.2312          & 186.12         & $>10^{3}$      & 0.4085          & 0.5178          & 0.2309          & 0.2747          & 0.2297          & 53.01          & 101.24         \\
& LoopQ           & \textbf{0.4088}                               & \textbf{0.5185}                                      & \textbf{0.2987}                               & \textbf{0.2782}                                      & \textbf{0.2480} & \textbf{32.85} & \textbf{63.71} & \textbf{0.4340} & 0.5288          & \textbf{0.3466} & \textbf{0.2892} & \textbf{0.2608} & \textbf{27.68} & \textbf{37.07} \\
\bottomrule
\end{tabular}}\end{table}

\vparagraph{Task Accuracy.}
Table~\ref{tab:main_results} reports W4A4 and W4A8 results across four backbones, with the largest gains under W4A4. On Ouro 1.4B, LoopQ improves five-task mean accuracy by $33.47$ percentage points over the strongest static baseline, retaining $92.23\%$ of BF16 accuracy. On Parcae, it improves LAMBADA accuracy by $16.98$ points. Across all four backbones under W4A4, the average relative gain in five-task mean accuracy is $67.8\%$. LAS and SLT reduce mismatch through loop-dependent scales and selective transforms, consistent with Proposition~\ref{prop:inter_layer}.

\vparagraph{Perplexity and Error Propagation.}
We focus on LAMBADA (\texttt{OP}) perplexity; WikiText-2 (\texttt{WT2}) shows a similar overall trend. On Ouro 1.4B, all static W4A4 baselines exceed $10^3$ OP PPL, while LoopQ achieves $8.20$. On Parcae, LoopQ reduces OP PPL by $86.0\%$ relative to the strongest static baseline. Across both datasets and all four backbones, a conservative estimate of the average relative PPL reduction is $75.2\%$, obtained by substituting $10^3$ for baseline PPL entries reported as $>10^3$ before averaging the eight relative reductions. TAC addresses this degradation by jointly calibrating boundary states and final predictions.

\vparagraph{Backbone-level Robustness.}
Recovery varies across architectures. Under W4A4, LoopFormer and Parcae retain $97.0\%$ and $92.8\%$ of BF16 mean accuracy, while both Ouro models show larger gaps despite substantial gains over static baselines. Ouro reuses its full Transformer stack across four loops, versus three shared layers in LoopFormer and four in Parcae. CTA applies corrections at loop boundaries, leaving more quantized layers between corrections on Ouro.

\vparagraph{Error along Loop Trajectories.}
Figure~\ref{fig:case} compares relative $\ell_2$ errors across layers and loops on LAMBADA. LoopQ maintains lower errors than the static baselines on both Ouro and Parcae, and reduces the pronounced peaks near loop transitions on Ouro. Consistent with Proposition~\ref{prop:inter_loop}, CTA corrects boundary states before reuse.

\begin{figure}[!t]\centering
\vspace{-3pt}
\includegraphics[width=\linewidth]{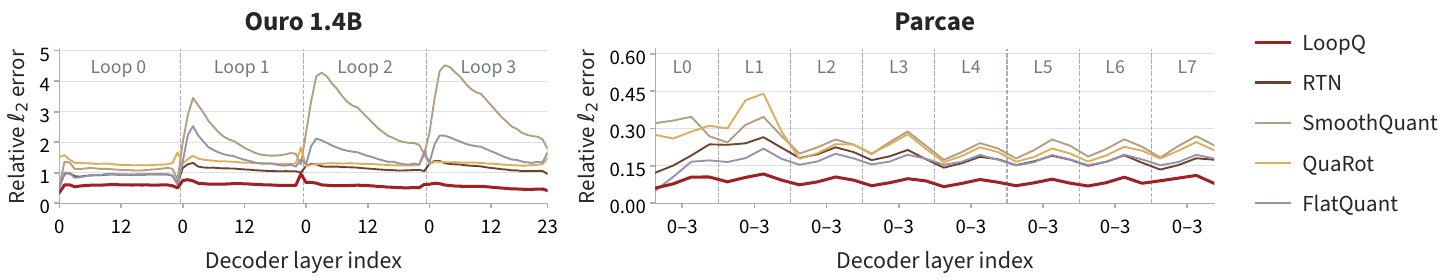}
\caption{Relative quantization error $\ell_2$ across layers and loops on LAMBADA.}
\label{fig:case}
\end{figure}

\subsection{Reasoning and Component Ablations}
\label{sec:experiment_analysis}

\begin{table}[!t]
\centering
\begin{minipage}[t]{0.65\linewidth}
\centering
\captionsetup{font=footnotesize,skip=3pt}
\caption{Ablations on Ouro 1.4B-Thinking under W4A4KV4. \textit{Unrolled FlatQuant} calibrates each of the four unrolled loops independently.}
\label{tab:ablation_summary}
\footnotesize
\setlength{\tabcolsep}{2pt}
\renewcommand{\arraystretch}{1.05}
\resizebox{\linewidth}{!}{\begin{tabular}{lrrrrrrrr}
\toprule
 & \multicolumn{6}{c}{\textbf{Task accuracy} $\uparrow$}
 & \multicolumn{2}{c}{\textbf{Perplexity} $\downarrow$} \\
\cmidrule(lr){2-7}\cmidrule(lr){8-9}
\textbf{Method} & \textbf{HS} & \textbf{WG} & \textbf{LB}
& \textbf{AC} & \textbf{MM} & \textbf{Mean}
& \textbf{WT2} & \textbf{OP} \\
\midrule
\textit{BF16}
& \textit{70.55} & \textit{66.06} & \textit{61.46}
& \textit{53.41} & \textit{68.86} & \textit{64.07}
& \textit{11.18} & \textit{6.22} \\
\textit{Unrolled FlatQuant}
& \textit{68.61} & \textit{65.04} & \textit{59.98}
& \textit{51.96} & \textit{61.57} & \textit{61.43}
& \textit{13.13} & \textit{6.72} \\
\midrule
QuaRot (KV16)
& 29.08 & 49.72 & 2.93 & 22.70 & 23.66 & 25.62
& 451.44 & $>10^{3}$ \\
LAS Only
& 52.13 & 55.09 & 18.86 & 38.91 & 27.73 & 38.54
& 37.89 & 143.87 \\
w/o CTA
& 66.14 & 59.91 & 55.21 & 48.63 & 57.41 & 57.46
& 14.03 & 9.59 \\
w/o Boundary Loss
& 66.43 & 61.40 & \textbf{57.48} & 47.70 & 59.03 & 58.41
& 13.51 & \textbf{8.84} \\
Random SLT
& 66.46 & 60.77 & 53.43 & 46.50 & 56.94 & 56.82
& 14.02 & 10.48 \\
Full LoopQ
& \textbf{66.89} & \textbf{62.04} & 56.61
& \textbf{49.74} & \textbf{59.51} & \textbf{58.96}
& \textbf{13.41} & 8.98 \\
\bottomrule
\end{tabular}}
\end{minipage}\hfill
\begin{minipage}[t]{0.33\linewidth}
\centering
\captionsetup{font=footnotesize,skip=3pt}
\caption{Deployment memory and five-task mean accuracy.}
\label{tab:deployment_accuracy}
\footnotesize
\setlength{\tabcolsep}{1.5pt}
\renewcommand{\arraystretch}{1.12}
\begin{tabular}{@{}clrr@{}}
\toprule
& \textbf{Method}
& \textbf{GiB} $\downarrow$
& \textbf{Mean} $\uparrow$ \\
\midrule
\multirow{3}{*}{\rotatebox[origin=c]{90}{\textit{BF16}}}
& Ouro 1.4B        & 2.67 & 64.07 \\
& Qwen3-1.7B       & 2.99 & 54.10 \\
& Qwen3-4B         & 6.98 & 63.18 \\
\midrule
\multirow{5}{*}{\rotatebox[origin=c]{90}{\textit{W4A4}}}
& \textit{Unrolled Ouro} & \textit{2.83} & \textit{61.43} \\
& Qwen3-1.7B       & 1.92 & 49.14 \\
& Qwen3-4B         & 3.38 & 60.00 \\
& Ouro FlatQuant   & 1.03 & 23.63 \\
& Ouro LoopQ       & \textbf{1.13} & \textbf{58.96} \\
\bottomrule
\end{tabular}
\end{minipage}
\end{table}

\begin{figure}[!t]
\centering
\captionsetup[subfigure]{skip=2pt}
\begin{subfigure}[t]{0.255\linewidth}
\centering
\includegraphics[height=2.8cm]{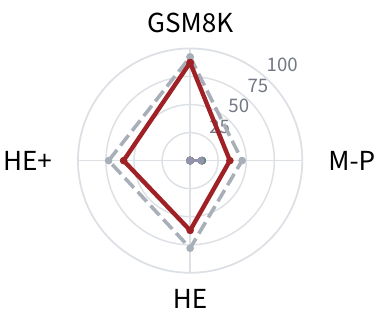}
\caption{Reasoning and code}
\label{fig:reasoning_tasks}
\end{subfigure}\hfill
\begin{subfigure}[t]{0.255\linewidth}
\centering
\includegraphics[trim={0 -12.32776519bp 0 0.00000000bp},height=2.8cm]{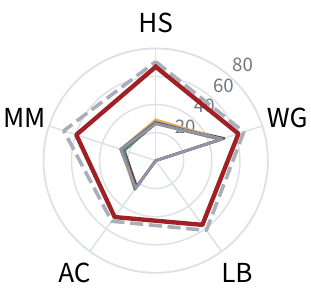}
\caption{KV4 task accuracy}
\label{fig:kv4_comparison}
\end{subfigure}\hfill
\begin{subfigure}[t]{0.33\linewidth}
\centering
\includegraphics[height=2.8cm]{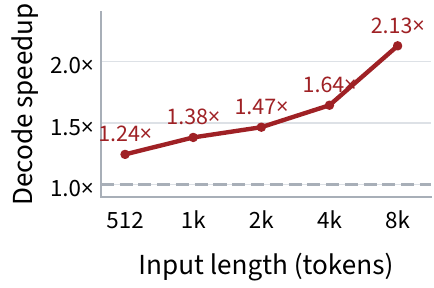}
\caption{Decode speedup}
\label{fig:decode_speedup}
\end{subfigure}\hfill
\begin{minipage}[t]{0.14\linewidth}
\includegraphics[width=\linewidth]{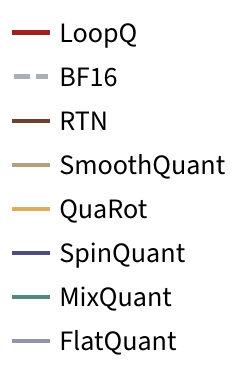}
\end{minipage}
\caption{Additional evaluations on Ouro 1.4B. (a) Accuracy on reasoning and code-generation tasks (\%). (b) Task accuracy (\%) under W4A4KV4. (c) Packed W4A4KV4 decode speedup over BF16.}
\label{fig:extended_evaluation}
\end{figure}

\vparagraph{Reasoning and Code Generation.}
We evaluate reasoning on GSM8K~\citep{cobbe2021training} and MMLU-Pro (\texttt{M-P})~\citep{wang2024mmlupro}, and code generation on HumanEval (\texttt{HE})~\citep{chen2021evaluating} and HumanEval+ (\texttt{HE+})~\citep{liu2023evalplus}. Figure~\ref{fig:reasoning_tasks} shows that most PTQ baselines struggle on these tasks, whereas LoopQ retains substantial capability. LoopQ achieves $87.3\%$ on GSM8K and $62.2\%$/$59.3\%$ on HE/HE+, compared with $92.3\%$ and $78.0\%$/$72.6\%$ for BF16. These tasks often require long autoregressive outputs, during which deviations in earlier predictions can alter the context for later tokens. This effect extends beyond the within-forward-pass error propagation characterized by Proposition~\ref{prop:inter_loop}.

\vparagraph{Ablation Study.}
Table~\ref{tab:ablation_summary} shows how LoopQ's components improve Ouro 1.4B-Thinking. LAS alone achieves $38.54\%$ five-task mean accuracy, compared with $25.62\%$ for the strongest listed static baseline. Adding SLT and trajectory-aware calibration, but not CTA, raises the mean to $57.46\%$ and reduces LAMBADA PPL from $143.87$ to $9.59$. Scale adaptation therefore helps but does not account for the full recovery. Adding CTA further improves all five task accuracies, bringing the mean to $58.96\%$. Selection and boundary alignment also matter: replacing GGN-based selection with six random groups lowers the mean by $2.14$ points, while removing the boundary loss lowers it by $0.55$ points despite improving some individual metrics. Appendix~\ref{apx:selection} shows the groups selected by SLT.

\vparagraph{Dynamic Quantization Settings.}
Table~\ref{tab:alternative_main} compares INT4 and NVFP4 under dynamic PTG quantization, with activation ranges estimated online in both settings. Under INT4, LoopQ exceeds FlatQuant by $3.81$ percentage points in mean accuracy. Under NVFP4, it achieves $0.6143$ mean accuracy, outperforming MixQuant by $0.28$ points and FlatQuant by $1.38$ points, with the lowest WikiText-2 word perplexity. MR-GPTQ~\citep{egiazarian2026bridging} is designed specifically for NVFP4; its INT4 entries are therefore N/A. Full results appear in Appendix~\ref{apx:alternative_quantization}.

\begin{table*}[t]
\centering\footnotesize
\caption{Alternative quantization results on Ouro 1.4B-Thinking.}
\label{tab:alternative_main}
\setlength{\tabcolsep}{3pt}
\renewcommand{\arraystretch}{1}
\resizebox{0.9\textwidth}{!}{\begin{tabular}{clrrrrrrr}
\toprule
Setting & Metric & RTN & MR-GPTQ & SmoothQuant & QuaRot & MixQuant & FlatQuant & LoopQ \\
\midrule
\multirow[c]{2}{*}{INT4} & Mean $\uparrow$ & 0.2555 & \multirow{2}{*}{N/A} & 0.2538 & 0.3722 & 0.2775 & 0.5461 & \textbf{0.5842} \\
 & WT2 PPL $\downarrow$ & $>10^3$ & & $>10^3$ & 62.89 & 225.54 & 18.84 & \textbf{16.63} \\
\midrule
\multirow[c]{2}{*}{NVFP4} & Mean $\uparrow$ & 0.5762 & 0.5962 & 0.5870 & 0.6053 & 0.6116 & 0.6005 & \textbf{0.6143} \\
 & WT2 PPL $\downarrow$ & 16.35 & 14.97 & 15.50 & 15.38 & 15.16 & 15.21 & \textbf{14.89} \\
\bottomrule
\end{tabular}}
\vspace{-3pt}
\end{table*}

\subsection{Deployment Efficiency}
\label{sec:deployment_efficiency}

\vparagraph{Model-State Storage.}
Table~\ref{tab:deployment_accuracy} shows that LoopQ adds $0.10$~GiB over Ouro FlatQuant while improving five-task mean accuracy by $35.33$ percentage points. Unrolled FlatQuant gains another $2.47$ points but requires $2.83$~GiB because it stores separate backbone copies. LoopQ achieves nearly the same accuracy at $1.13$~GiB, $57.68\%$ less than Ouro BF16. It also reports higher accuracy and a smaller footprint than quantized Qwen3-1.7B, although Qwen3 memory is measured after loading and Ouro memory is estimated. Appendix~\ref{apx:storage_accuracy} details these memory-footprint accounting conventions.

\vparagraph{KV-Cache Quantization.}
Under W4A4KV4 on Ouro 1.4B-Thinking, static baselines achieve at most $0.02\%$ LAMBADA accuracy, whereas LoopQ reaches $56.61\%$ and achieves $58.96\%$ five-task mean accuracy (Figure~\ref{fig:kv4_comparison}). Its WikiText-2 and LAMBADA PPL remain at $13.41$ and $8.98$, respectively, showing that LoopQ's recovery extends to quantized caches.

\vparagraph{Decode Throughput.}
Figure~\ref{fig:decode_speedup} reports packed W4A4KV4 inference on an RTX 5090 with $1{,}024$ generated tokens per request. LoopQ achieves $1.24\times$ BF16 decode throughput at $512$ prompt tokens and $2.13\times$ at $8{,}192$, where both runtimes slow down but the relative advantage grows. These are achieved-throughput comparisons under the stated configuration; Appendix~\ref{apx:packed_runtime} reports the measurement protocol and component-level overhead.

\section{Related Work}
\label{sec:related}

\vparagraph{PTQ for LLMs.}
Post-training quantization (PTQ) reduces model storage and inference data movement without end-to-end retraining~\citep{alizadeh2020gradient,chmiel2020robust,xiao2023robustmq,sun2025flatquant,liu2025spinquant,egashira2024exploiting,lin2024awq,frantar2022gptq}. Weight-only methods limit quantization error using approximate second-order information~\citep{frantar2022gptq} or protect weights associated with salient activations~\citep{lin2024awq}, but leave activations at higher precision. Joint weight-activation methods address activation outliers by retaining selected dimensions at higher precision~\citep{dettmers2022gpt3}, allocating precision by salience~\citep{huang2024slim}, or preserving high-variance subspaces~\citep{saxena2025resq}. Transformation-based methods reshape activations through channel-wise scaling~\citep{xiao2023smoothquant}, rotations~\citep{ashkboos2024quarot,liu2025spinquant}, or more flexible transformations~\citep{hu2025ostquant,sun2025flatquant,sanjeet2026mixquant}. LoopQ retains a shared quantized backbone while adapting selected scales and transformations across loops and calibrating along the recurrent trajectory.

\vparagraph{Looped Language Models.}
Transformer compression has used pruning, merging, and layer reorganization~\citep{imfeld2023transformer,liu2024pruning,zhou2023modular,verma2025merging}. LoopLMs instead increase computational depth by repeatedly applying shared layers. Universal Transformer introduced recurrent Transformer computation~\citep{dehghani2018universal}, and ALBERT demonstrated extensive cross-layer parameter sharing~\citep{lan2019albert}. Recent architectures explore varied forms of recursion~\citep{bae2025mixture,zhu2025scaling,zeitoun2026hyperloop,prairie2026parcae,jeddi2026loopformer}, while inference systems exploit weight reuse under memory and bandwidth constraints~\citep{zhou2022transpim,laguna2022hardware,kwon2024lol,gupta2026persistent,riera2022crew}. These studies primarily address architecture, parameter efficiency, and full-precision behavior; low-bit quantization of recurrent hidden-state trajectories remains less explored.

\section{Conclusion}
\label{sec:conclusion}

In this paper, we present \textbf{LoopQ}, a loop-aware PTQ framework for LoopedLMs. We show that LoopLMs face unique quantization challenges caused by loop-dependent distribution shifts, transition-state reuse, and recursive error accumulation. LoopQ addresses these challenges by preserving a shared quantized backbone while adding lightweight loop-dependent adaptations, including activation scaling, selective transformation, cross-loop state alignment, and trajectory-aware calibration. Experiments across multiple LoopLM backbones and benchmarks demonstrate that LoopQ consistently improves low-bit robustness over static PTQ baselines, especially under  W4A4 quantization. LoopQ also preserves early-exit capability and achieves up to $2.13\times$ BF16 decode throughput in packed deployment.

\section*{AI Use Statement}
We used generative AI tools to assist with implementing portions of the proposed method and with editing the manuscript to improve readability and language clarity. All AI-assisted code was reviewed, verified, and tested by the authors, and all AI-assisted edits to the manuscript were reviewed for correctness. We take full responsibility for the final content of this work, including all text, claims, code, and other artifacts produced with the aid of generative AI.

\section*{Ethics Statement}
This work studies post-training quantization of existing language models using public calibration and evaluation datasets; it does not involve new human-subject data collection. Quantization can change model outputs, including behavior relevant to bias, harmful content, and privacy. Our accuracy and efficiency results do not establish that such behavior is preserved. Deployments of quantized checkpoints should therefore evaluate these risks separately.

\section*{Reproducibility Statement}
We specify LoopQ's quantization and calibration objectives in the main text and provide implementation details, including group selection and the transition adapter, in Appendices~\ref{apx:implementation_details}--\ref{apx:adaptive_weighting}. Proofs appear in Appendix~\ref{apx:proof}. The experimental sections and appendices document the calibration data, quantization settings, task-level results, and protocols for measuring deployment memory and inference throughput.

\bibliographystyle{iclr2027_conference}
\bibliography{reference}

\clearpage

\appendix
\setcounter{secnumdepth}{3}
\section{Notation and Theoretical Analysis}
\label{apx:proof}

\subsection{Notation Table}
\label{app:notation}

\begin{table}[!t]\centering\small
    \caption{Summary of notation.}\label{tab:notation}
    \begin{tabularx}{\linewidth}{@{}p{0.29\linewidth}X@{}}\toprule
        Notation                                       & Meaning                                                                                                                  \\\midrule
        $\vx,\train$                                   & Input token sequence and calibration set.                                                                                \\
        $T,L,n,d,v$                                    & Number of loops, shared layers, tokens, hidden features, and vocabulary entries.                                         \\
        $t,\ell,i,j$                                   & Loop, layer (or selectable group), token, and coordinate indices.                                                        \\
        $f_\ell(\cdot;\vtheta_\ell)$                   & Full-precision layer and its parameter vector.                                                                           \\
        $\mH_{t,\ell},\widetilde\mH_{t,\ell}$          & Teacher and quantized hidden-state matrices.                                                                             \\
        $\mH_t=\mH_{t,0}$                              & Input-state matrix of loop $t$.                                                                                          \\
        $\mX_{t,\ell}(\vx;\mP)$                        & Transformed teacher activation $\mH_{t,\ell-1}(\vx)\mP$.                                                                 \\
        $(\mH_{t,\ell})_{i,:}$                         & Token row $i$ of $\mH_{t,\ell}$.                                                                                         \\
        $F,\widetilde F_t$                             & Teacher loop map and quantized loop map.                                                                                 \\
        $\mZ,\widetilde\mZ$                            & Teacher and student terminal logit matrices.                                                                             \\
        $\mW_\ell,\mP_\ell$                            & Frozen backbone weight and shared invertible transform.                                                                  \\
        $Q_a,Q_w$                                      & Activation and weight quantization--dequantization maps.                                                                 \\
        $c_\ell,c_{t,\ell}$                            & Shared and loop-dependent static activation quantization steps.                                                          \\
        $q_{\min},q_{\max}$                            & Minimum and maximum integer quantization levels.                                                                         \\
        $\varepsilon_t,\varepsilon_t^{\mathrm{quant}}$ & Propagated state error and local quantization error.                                                                     \\
        $\gamma_t$                                     & Local Lipschitz factor of the teacher loop map.                                                                          \\
        $\mP_{t,\ell},\mDelta_{t,\ell}$                & Loop-dependent transform and its perturbation from a shared transform.                                                   \\
        $\vphi_{\ell j}$                       & Occurrences of rotation coordinate $j$ of group $\ell$ across loops.                                             \\
        $S_\ell$                                       & Sharing-sensitivity score for selecting loop-dependent transforms.                                                       \\
        $A_t$                                          & Adapter between loops $t$ and $t+1$.                                                                                     \\
        $\Delta\widetilde\mH_t$                & Current loop update.                                                                                             \\
        $\mU,\mV,\vb_t$                        & Shared low-rank factors of sizes $d\times r$ and $2d\times r$, and loop-dependent latent offset.                 \\
        $\Ls$                                          & Unified calibration objective.                                                                                           \\
        $\KL(\mZ\|\widetilde\mZ)$                      & Mean predictive KL associated with the terminal logits.                                                                  \\
        $k_i,e_{t,i}$                          & Token-level terminal KL and calibrated relative boundary error.                                                  \\
        $\mu_t$                                & Effective adaptive coefficient for the matrix-level boundary loss.                                               \\
        $\mG_{\ell j},\mC_{\ell j}$            & Cross-loop GGN block and its damped estimate.                                                                    \\
        $\vp_i,\mF_i$                          & Quantized predictive distribution and Fisher information matrix at token $i$.                                    \\
        $g_t,\mR_t,\Pi^\perp$                  & Input-dependent gate, scaled bottleneck correction, and projection removing the component along the loop update. \\\bottomrule
    \end{tabularx}\end{table}

\subsection{Proof of Proposition~\ref{prop:inter_layer}}
\begin{customprop}{\ref{prop:inter_layer}}
For $\vx\sim\train$, let $\mH_{t,\ell-1}(\vx)$ be the hidden state on the fixed full-precision trajectory, and define $\mX_{t,\ell}(\vx;\mP):=\mH_{t,\ell-1}(\vx)\mP$. With $Q_a$ from Eq.~\eqref{eq:quant_a} applied elementwise, the expected transformed-activation quantization error is
\begingroup\small
\begin{equation*}
    \varepsilon_{t,\ell}(\mP,c)
    :=\E_{\vx\sim\train}\!\left[
        \left\|Q_a(\mX_{t,\ell}(\vx;\mP);c)
        -\mX_{t,\ell}(\vx;\mP)\right\|_F^2
    \right].
\end{equation*}
\endgroup
Assume that all minima below are finite and attained over the same admissible set of invertible $\mP$ and $c>0$. Let $\varepsilon_{t,\ell}^{*}:=\min_{\mP,c}\varepsilon_{t,\ell}(\mP,c)$ and define the cost of sharing one configuration across loops as
\begingroup\small
\begin{equation*}
    \Delta_\ell^{\mathrm{share}}
    :=\min_{\mP_\ell,c_\ell}
        \sum_{t=0}^{T-1}\varepsilon_{t,\ell}(\mP_\ell,c_\ell)
      -\sum_{t=0}^{T-1}\varepsilon_{t,\ell}^{*}.
\end{equation*}
\endgroup
Then $\Delta_\ell^{\mathrm{share}}\ge0$, with equality if and only if a single configuration is optimal for every loop. In particular, if any two loops have disjoint sets of optimal configurations, then $\Delta_\ell^{\mathrm{share}}>0$: every shared configuration incurs excess transformed-activation quantization error in at least one loop.

\end{customprop}
\par\begingroup
    \begin{proof}
        We prove the claim by contradiction.

        For any shared configuration $(\mP_\ell,c_\ell)$, the loop-wise activation quantization error satisfies
        \[
            \varepsilon_{t,\ell}(\mP_\ell,c_\ell)\ge\varepsilon_{t,\ell}^{*}.
        \]
        Summing over loops and minimizing over the shared configuration gives $\Delta_\ell^{\mathrm{share}}\ge0$.

        Suppose that shared calibration incurs zero total excess error. Let $(\mP_\ell^{*},c_\ell^{*})$ attain the shared minimum. Then
        \[
            \sum_{t=0}^{T-1}\left[
                \varepsilon_{t,\ell}(\mP_\ell^{*},c_\ell^{*})-\varepsilon_{t,\ell}^{*}
                \right]=0.
        \]
        Since every summand is nonnegative, each must vanish. Thus, $(\mP_\ell^{*},c_\ell^{*})$ is simultaneously optimal for all loops. In particular, if two loops have disjoint sets of optimal configurations, this contradicts their incompatibility. Conversely, a common optimal configuration attains every loop-wise minimum and therefore gives $\Delta_\ell^{\mathrm{share}}=0$.

        Consequently, when no common optimum exists, every shared configuration satisfies
        \[
            \sum_{t=0}^{T-1}\left[
                \varepsilon_{t,\ell}(\mP_\ell,c_\ell)-\varepsilon_{t,\ell}^{*}
                \right]\ge\Delta_\ell^{\mathrm{share}}>0.
        \]
        At least one loop must therefore incur strictly positive excess error. The proposition follows.
    \end{proof}

    \paragraph{Implications for loop-dependent calibration.}
    To relate this excess activation error to loop-dependent activation distributions, we first consider scale drift. Fix a shared transformation $\mP_\ell$ and suppose that
    \[
        \mX_{t,\ell}(\vx;\mP_\ell)\overset{d}{=}s_t\rmR,\qquad s_t>0,
    \]
    where $\vx\sim\train$ and $\rmR$ has a common normalized distribution across loops with a finite second moment. For a symmetric uniform quantizer with fixed integer bounds,
    \[
        Q_a(s\mX;c)=sQ_a(\mX;c/s),\qquad s>0.
    \]
    Therefore,
    \[
        \varepsilon_{t,\ell}(\mP_\ell,c)
        =s_t^2\E\!\left[\left\|Q_a(\rmR;c/s_t)-\rmR\right\|_F^2\right]
        =s_t^2\varepsilon_{\mathrm{ref}}(c/s_t),
    \]
    where $\varepsilon_{\mathrm{ref}}(\alpha):=\E[\|Q_a(\rmR;\alpha)-\rmR\|_F^2]$. Thus, optimizing $c$ for loop $t$ is equivalent to optimizing the normalized step size $\alpha=c/s_t$.

    Assume that $\varepsilon_{\mathrm{ref}}$ has a unique minimizer $\alpha^*>0$. The optimal step size for loop $t$ is then
    \[
        c_{t,\ell}^{*}=s_t\alpha^*.
    \]
    If a shared step size $c_\ell$ were simultaneously optimal for two loops $t\neq t'$, it would have to satisfy
    \[
        c_\ell=s_t\alpha^*=s_{t'}\alpha^*.
    \]
    Since $\alpha^*>0$, this equality implies $s_t=s_{t'}$, contradicting the scale-drift condition. Hence, under these assumptions, distinct activation scales yield incompatible optimal step sizes. Applying the proposition to calibration with $\mP_\ell$ fixed gives strictly positive excess activation error whenever the shared minimum is attained.

    We next consider directional mismatch. Hold the activation scales fixed and let $q_\ell$ be a positive-definite quadratic surrogate for the recurrent calibration objective over loop-dependent transforms. Denote its unrestricted minimizer by $\{\mP_{t,\ell}^{*}\}$. Write $\mD_t:=\mP_{t,\ell}-\mP_{t,\ell}^{*}$. Completing the square gives
    \[
        q_\ell(\{\mP_{t,\ell}\})
        =q_\ell(\{\mP_{t,\ell}^{*}\})+\frac12\sum_{t,t'=0}^{T-1}
        \left\langle
        \mD_t,
        \mathcal C_\ell^{t,t'}[\mD_{t'}]
        \right\rangle_F,
    \]
    where the blocks $\mathcal C_\ell^{t,t'}$ form a symmetric positive-definite curvature operator. The off-diagonal blocks account for interactions between loops. Under the sharing constraint $\mP_{t,\ell}=\mP_\ell$ for every $t$, write $\mD_t(\mP):=\mP-\mP_{t,\ell}^{*}$. The excess loss becomes
    \[
        \Delta_\ell^{\mathrm{quad}}
        =\min_{\mP_\ell}\frac12\sum_{t,t'=0}^{T-1}
        \left\langle
        \mD_t(\mP_\ell),
        \mathcal C_\ell^{t,t'}[\mD_{t'}(\mP_\ell)]
        \right\rangle_F.
    \]
    Positive definiteness implies that this minimum is zero if and only if a single $\mP_\ell$ equals every $\mP_{t,\ell}^{*}$. If two loops favor different transforms, no shared transform can satisfy this condition, and $\Delta_\ell^{\mathrm{quad}}>0$.

    Thus, scale drift can produce incompatible optimal step sizes, while conflicting transformation updates incur a local cost when constrained to be shared. The latter cost provides the basis for SLT's sharing-sensitivity score, computed in rotation coordinates using a damped GGN approximation (Appendix~\ref{apx:slt_optimization}). The loss increase is exact for the quadratic surrogate; its use for the recurrent calibration loss is a local approximation.
\par\endgroup

\subsection{Proof of Proposition~\ref{prop:inter_loop}}

\begin{customprop}{\ref{prop:inter_loop}}
Define the loop error as $\varepsilon_t:=\|\widetilde{\mH}_t-\mH_t\|_F$. Assume that both recurrences start from the same initial state and, for each loop $t$, the full-precision function $F$ is $\gamma_t$-Lipschitz on the segment between $\widetilde{\mH}_t$ and $\mH_t$. Then $\varepsilon_{t+1}\le\varepsilon_t^{\mathrm{quant}}+\gamma_t\varepsilon_t$, where $\varepsilon_t^{\mathrm{quant}}:=\|\widetilde F_t(\widetilde{\mH}_t)-F(\widetilde{\mH}_t)\|_F$. Unrolling this recursion yields
\begingroup\small\begin{equation*}
	\varepsilon_T\le\sum_{\tau=0}^{T-1}\left(\prod_{t=\tau+1}^{T-1}\gamma_t\right)\varepsilon_\tau^{\mathrm{quant}}.
\end{equation*}\endgroup
Thus, each loop-wise quantization error is propagated through downstream loop sensitivities, with earlier errors passing through a longer chain of transformations.

\end{customprop}
\begin{proof}
	By definition,
	\[
		\varepsilon_{t+1}=\|\widetilde{\mH}_{t+1}-\mH_{t+1}\|_F
		=\|\widetilde F_t(\widetilde{\mH}_t)-F(\mH_t)\|_F.
	\]
	Applying the triangle inequality gives
	\[
		\varepsilon_{t+1}\le\|\widetilde F_t(\widetilde{\mH}_t)-F(\widetilde{\mH}_t)\|_F
		+\|F(\widetilde{\mH}_t)-F(\mH_t)\|_F.
	\]
	By the definition of the loop-wise quantization error,
	\[
		\|\widetilde F_t(\widetilde{\mH}_t)-F(\widetilde{\mH}_t)\|_F=\varepsilon_t^{\mathrm{quant}}.
	\]
	Since $F$ is $\gamma_t$-Lipschitz on the segment between $\widetilde{\mH}_t$ and $\mH_t$, we have
	\[
		\|F(\widetilde{\mH}_t)-F(\mH_t)\|_F\le\gamma_t\|\widetilde{\mH}_t-\mH_t\|_F=\gamma_t\varepsilon_t.
	\]
	Combining the two bounds yields
	\[
		\varepsilon_{t+1}\le\varepsilon_t^{\mathrm{quant}}+\gamma_t\varepsilon_t.
	\]
	Unrolling this recursion from $0$ to $T-1$ gives
	\[
		\varepsilon_T\le\left(\prod_{t=0}^{T-1}\gamma_t\right)\varepsilon_0
		+\sum_{\tau=0}^{T-1}\left(\prod_{t=\tau+1}^{T-1}\gamma_t\right)\varepsilon_\tau^{\mathrm{quant}},
	\]
	where the empty product is defined as $1$. Since the full-precision and quantized recurrences start from the same initial state, $\widetilde{\mH}_0=\mH_0$, we have $\varepsilon_0=0$. Therefore,
	\[
		\varepsilon_T\le\sum_{\tau=0}^{T-1}\left(\prod_{t=\tau+1}^{T-1}\gamma_t\right)\varepsilon_\tau^{\mathrm{quant}}.
	\]
	Thus, the quantization error introduced at loop $\tau$ is propagated through all downstream sensitivity factors $\gamma_{\tau+1},\ldots,\gamma_{T-1}$. Earlier errors pass through a longer chain of loop transformations and can therefore have a larger effect on the final state. The proposition follows.
\end{proof}

\FloatBarrier
\section{Implementation and Calibration Details}
\label{apx:detailed_implementation}
\label{apx:loopwise_details}

\subsection{Experimental and Calibration Setup}
\label{apx:implementation_details}

\vparagraph{Calibration Data.}
We construct a calibration pool with equal token contributions from the Pile~\citep{gao2020pile} and public reasoning datasets: OpenThoughts3-1.2M, AceReason-1.1-SFT, OpenCodeReasoning, Llama-Nemotron-Post-Training-Dataset, and DeepWriting-20K. We sample $1{,}024$ training examples and $128$ validation examples and freeze backbone weights during calibration. All LoopLM PTQ baselines collect activations across the full recursive forward pass, and methods within each comparison use the same calibration data.

\vparagraph{Quantization Settings.}
Static quantization uses groups of 32 for weights and activations, with loop-specific activation scales in LoopQ. Dynamic quantization uses per-output-channel weights and per-token activation groups of 128. KV4 uses asymmetric online quantization with groups of 32 in the static setting and up to 128 per attention head in the dynamic setting.

\subsection{Loop-Aware Activation Scale Calibration}
\label{apx:las_details}

Each linear-layer activation quantizer maintains a scale for physical loop $t$ and feature group $g$:
\begin{equation}
\label{eq:las_effective_scale}
c_{t,\ell,g}=c^0_{t,\ell,g}\exp(\rho_{t,\ell,g}),
\end{equation}
where $c^0_{t,\ell,g}>0$ is the calibrated base scale and $\rho_{t,\ell,g}$ is a trainable log-scale correction initialized to zero. Starting from the FlatQuant MSE checkpoint, we refine the base scales by minimizing activation-reconstruction MSE under quantized forward passes. The base scales are then fixed while the corrections are optimized jointly with the other LoopQ parameters.

Scales are indexed by physical loop and feature group, independently of SLT selection. At export, $\exp(\rho_{t,\ell,g})$ is folded into the stored scale for retrieval at inference.

\subsection{Loop-Dependent Transformation Selection}
\label{apx:slt_optimization}

\vparagraph{Candidate Groups and Parameterization.}
\label{apx:slt_groups}
Selection operates on transformation groups. Linear-path groups couple a transform with its corresponding linear weights, while KV quantization adds key-cache transformation sites (Table~\ref{tab:loopwise_block_expansion}). Selected linear-path groups require loop-dependent weight representations at deployment; key-cache sites act on the cache path. Following FlatQuant~\citep{sun2025flatquant}, we use Kronecker-factored transforms. Only the effective Cayley coordinates of the orthogonal factors vary across loops; singular values and optional channel scales remain shared, and backbone weights stay frozen. Let $\vphi_{\ell j}\in\R^T$ collect coordinate $j$ of group $\ell$ across loops.

\begin{table}[!t]
\centering
\small
\caption{Transformation groups used as the minimum selection units.}
\label{tab:loopwise_block_expansion}
\setlength{\tabcolsep}{4pt}
\renewcommand{\arraystretch}{1.12}
\begin{tabularx}{\linewidth}{@{}ll
    >{\raggedright\arraybackslash\footnotesize\ttfamily}X
    >{\raggedright\arraybackslash\footnotesize\ttfamily}X@{}}
\toprule
Model & Candidate group
    & \multicolumn{1}{l}{Weight side}
    & \multicolumn{1}{l@{}}{Transform side} \\
\midrule
Ouro & QKV input & q\_proj, k\_proj, v\_proj & self\_attn.ln\_trans \\
     & Attention output & o\_proj & self\_attn.o\_trans \\
     & MLP gate/up & gate\_proj, up\_proj & mlp.up\_gate\_trans \\
     & MLP down & down\_proj & mlp.down\_trans \\
     & Key cache & {\normalfont ---} &  key-cache transform \\
\midrule
LoopFormer & QKV input & attn.c\_attn & attn.in\_trans \\
           & Attention output & attn.c\_proj & attn.out\_trans \\
           & MLP up & mlp.c\_fc & mlp.up\_trans \\
           & MLP down & mlp.c\_proj & mlp.down\_trans \\
\midrule
Parcae & QKV input & attn.c\_qkv & attn.in\_trans \\
       & Attention output & attn.c\_proj & attn.out\_trans \\
       & MLP up & mlp.fc & mlp.up\_trans \\
       & MLP down & mlp.proj & mlp.down\_trans \\
\bottomrule
\end{tabularx}
\end{table}

\vparagraph{Curvature Estimation and Sharing-Sensitivity Score.}
\label{apx:slt_score}
We estimate cross-loop curvature blocks with output probes whose covariance matches the predictive Fisher. For student probabilities $\vp_i$, sample zero-mean, identity-covariance noise $\boldsymbol\epsilon_i$ and construct
\[
\vz_i
=\sqrt{\vp_i}\odot\boldsymbol\epsilon_i
-\vp_i\bigl(\sqrt{\vp_i}^{\top}\boldsymbol\epsilon_i\bigr).
\]
When sampling $k$ of $M$ microbatches uniformly without replacement, each vector--Jacobian product (VJP) is scaled by $\sqrt{M/(Nk)}$, where $N$ counts valid target tokens in the full optimizer update. We average the scaled VJP outer products over independent probes within each sampled microbatch, then sum over sampled microbatches. Teacher-selected support is fixed during differentiation, quantization uses straight-through derivatives, and estimated blocks are reused between refreshes. The approximation omits cross-coordinate and cross-module coupling.

Following the second-order compression principle of Optimal Brain Surgeon~\citep{hassibi1992second}, we estimate the local cost of tying parameters across loops. With damped curvature, let
\[
\mC_{\ell j}=\widehat{\mG}_{\ell j}+\nu_s\mI,
\qquad
\vy_{\ell j}=\vphi_{\ell j}-\mC_{\ell j}^{-1}\vg_{\ell j},
\]
where $\vg_{\ell j}$ is the clipped calibration gradient. The shared coordinate minimizing the local quadratic surrogate and the resulting sharing-sensitivity score are
\begin{equation}
\label{eq:sharing_score}
a_{\ell j}
=\frac{\vone^\top\mC_{\ell j}\vy_{\ell j}}
{\vone^\top\mC_{\ell j}\vone},
\qquad
S_\ell
=\frac12\sum_j
\|\vy_{\ell j}-a_{\ell j}\vone\|_{\mC_{\ell j}}^2,
\end{equation}
where $\|\vv\|_{\mC}^2=\vv^\top\mC\vv$. A larger $S_\ell$ indicates a greater predicted loss from enforcing sharing. We set $\nu_s=\max(\nu_{\min},\|\vg\|_2/r_{\rm trial})$, where $\vg$ concatenates candidate rotation gradients, to bound the unconstrained trial step before sharing projection.

\vparagraph{Selection Schedule.}
\label{apx:slt_schedule}
After warmup, we gradually reduce the loop-dependent group budget, retaining groups with the highest sharing-sensitivity scores. Groups remain eligible for reselection until the structure is fixed; calibration then continues with that configuration. For Ouro 1.4B-Thinking, the final budget is six of 120 candidate groups. Appendix~\ref{apx:selection} compares the resulting selections for Ouro 1.4B and Ouro 2.6B.

\subsection{Cross-Loop Transition Adapter Implementation}
\label{apx:cta_details}

Let $\mathcal N(\cdot)$ denote group-wise RMS normalization. The low-rank correction is
\begin{equation}
\label{eq:cta_low_rank}
\begin{aligned}
\mR_t
={}&\operatorname{RMS}(\Delta\widetilde{\mH}_t)\odot\\
&\left[
\operatorname{SiLU}\!\left(
[\mathcal N(\widetilde{\mH}_{t,0}),
 \mathcal N(\Delta\widetilde{\mH}_t)]\mV+\vb_t
\right)\mU^\top
\right],
\end{aligned}
\end{equation}
where $\mV\in\R^{2d\times r}$ and $\mU\in\R^{d\times r}$ are shared across loop boundaries, while $\vb_t$ is boundary-specific. The RMS factor is broadcast within each feature group, and fixed output scaling is absorbed into $\mU$.

The input-conditioned tanh gate $g_t$ uses a shared MLP with boundary-specific biases. Its inputs are the relative update RMS, cosine similarities between the update and the loop input and output, and the normalized loop index. The projection $\Pi^\perp_{\Delta\widetilde{\mH}_t}$ is applied per token and feature group, with the squared update norm clamped below by $\epsilon$ for numerical stability.

\subsection{Adaptive Trajectory Weighting}
\label{apx:adaptive_weighting}

For each valid prediction token $i$, we define the predictive KL and relative boundary error as
\begin{equation}
\label{eq:tac_h}
k_i=\KL(\mZ_{i,:}\|\widetilde\mZ_{i,:}),\qquad
e_{t,i}=
\frac{\|(\widetilde\mH_{t+1,0}-\mH_{t+1,0})_{i,:}\|_2^2}
{\|(\mH_{t+1,0})_{i,:}\|_2^2+10^{-6}}.
\end{equation}
The KL compares teacher and student probabilities obtained from the terminal logits and renormalized on the same teacher-selected support. The boundary error compares the post-adapter student state with the corresponding teacher state.

At each boundary, we normalize $e_{t,i}$ by the median ratio of relative boundary error to $\max(k_i,0)+10^{-8}$ over valid tokens in the initial calibration microbatches. This normalization remains fixed throughout training; $e_{t,i}$ below denotes the normalized error. The effective coefficient in Eq.~\eqref{eq:loopq_unified_loss} satisfies
\begin{equation}
\label{eq:tac_mu}
\mu_t\|\widetilde\mH_{t+1,0}-\mH_{t+1,0}\|_F^2
=\frac{1}{n(T-1)}\sum_{i=1}^{n}
\frac{k_i^+e_{t,i}}{e_{t,i}+k_i^+},
\end{equation}
where $k_i^+=\max(k_i,0)$ removes small negative KL values caused by numerical error, and $n$ counts valid prediction tokens. A summand is defined as zero when both errors vanish, and $\mu_t=0$ when the boundary states match. Gradients flow through the adaptive weights. All loss reductions exclude padding and non-target prefixes.

\FloatBarrier
\section{Additional Experiments and Analysis}
\label{apx:additional_evaluation}

\subsection{Allocation of Loop-Dependent Transformation Groups}
\label{apx:selection}

Figure~\ref{fig:loopq_analysis} shows the groups selected for loop-dependent transformations in Ouro 1.4B and Ouro 2.6B. Each shared layer has five candidates: QKV input, attention output, MLP gate/up, MLP down, and key cache. Under a 5\% group budget, the models select 6 of 120 and 12 of 240 candidates, respectively. Selection uses the sharing-sensitivity scores in Appendix~\ref{apx:slt_score}, which estimate the calibration cost of tying a group's transformations across loops.

Linear-path selections concentrate near the end of the shared stack. In Ouro 1.4B, all four selected linear-path groups are at the final layer, covering QKV input, attention output, and both MLP groups. In Ouro 2.6B, eight linear-path groups span layers 40--47: four QKV inputs, two attention outputs, and both MLP groups at the final layer. This concentration is consistent with the pronounced error spikes near loop transitions in Figure~\ref{fig:case}. Under recursive reuse, the final layer of one loop directly supplies the state processed by the first layer of the next, making these boundary layers particularly influential in how quantization errors enter and propagate through subsequent loops. Adapting late-layer transformations can therefore reduce the mismatch passed across loop boundaries.

Key-cache selections are more dispersed, appearing at layers 3 and 16 in Ouro 1.4B and layers 0, 3, 24, and 41 in Ouro 2.6B. They account for one-third of each model's selected groups. This difference suggests that useful selection sites depend on group type as well as depth; restricting selection to late layers would exclude the earlier key-cache sites.

\begin{figure}[t]
\centering
\includegraphics[width=\textwidth]{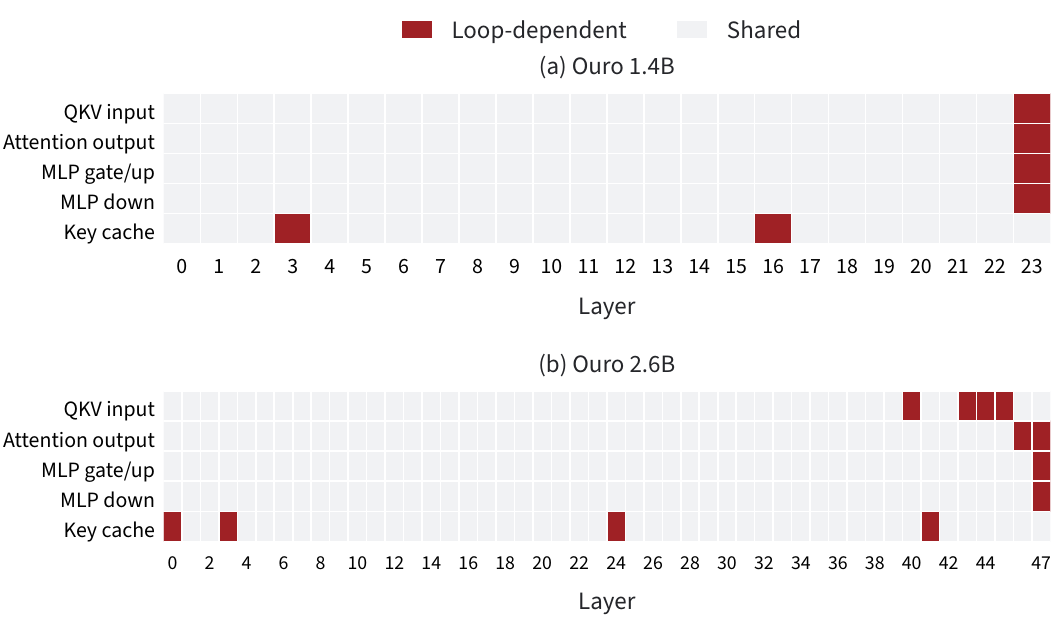}
\caption{Final loop-dependent group selections for (a) Ouro 1.4B and (b) Ouro 2.6B. Red indicates independent transformations for the four loops.}
\label{fig:loopq_analysis}
\end{figure}

\subsection{Dynamic Quantization Settings}
\label{apx:alternative_quantization}

Table~\ref{tab:alternative_full} provides the task-level results summarized in Table~\ref{tab:alternative_main} for INT4 and NVFP4 under dynamic per-token-per-group (PTG) quantization on Ouro 1.4B-Thinking. Both formats estimate activation ranges online. MR-GPTQ is designed specifically for NVFP4; its INT4 entries are therefore N/A.

\begin{table*}[t]
\centering\footnotesize
\captionsetup{font=footnotesize}
\caption{Task-level results for INT4 and NVFP4 under dynamic PTG quantization on Ouro 1.4B-Thinking. HS, WG, LB, AC, and MM denote HellaSwag, WinoGrande, LAMBADA, ARC-C, and MMLU accuracy on a 0--1 scale. Mean averages these five tasks. WT2 denotes WikiText-2 word perplexity. Best values within each format are bolded.}
\label{tab:alternative_full}
\setlength{\tabcolsep}{3pt}
\renewcommand{\arraystretch}{1.05}
\resizebox{\textwidth}{!}{\begin{tabular}{clrrrrrrr}
\toprule
Setting & Metric & RTN & MR-GPTQ & SmoothQuant & QuaRot & MixQuant & FlatQuant & LoopQ \\
\midrule
\multirow[c]{7}{*}{INT4} & HS $\uparrow$ & 0.2627 & \multirow{7}{*}{N/A} & 0.2650 & 0.4580 & 0.3215 & 0.6427 & \textbf{0.6612} \\
 & WG $\uparrow$ & 0.5004 & & 0.4893 & 0.4996 & 0.4886 & 0.6125 & \textbf{0.6219} \\
 & LB $\uparrow$ & 0.0000 & & 0.0000 & 0.2272 & 0.0749 & 0.4820 & \textbf{0.5544} \\
 & AC $\uparrow$ & 0.2526 & & 0.2543 & 0.3541 & 0.2577 & 0.4735 & \textbf{0.4974} \\
 & MM $\uparrow$ & 0.2619 & & 0.2602 & 0.3222 & 0.2449 & 0.5197 & \textbf{0.5862} \\
 & Mean $\uparrow$ & 0.2555 & & 0.2538 & 0.3722 & 0.2775 & 0.5461 & \textbf{0.5842} \\
 & WT2 PPL $\downarrow$ & $>10^3$ & & $>10^3$ & 62.89 & 225.54 & 18.84 & \textbf{16.63} \\
\midrule
\multirow[c]{7}{*}{NVFP4} & HS $\uparrow$ & 0.6656 & 0.6714 & 0.6809 & 0.6874 & 0.6842 & 0.6703 & \textbf{0.6880} \\
 & WG $\uparrow$ & 0.6014 & 0.6243 & 0.6196 & 0.6281 & \textbf{0.6361} & 0.6267 & 0.6350 \\
 & LB $\uparrow$ & 0.5317 & 0.5721 & 0.5519 & 0.5732 & 0.5844 & 0.5614 & \textbf{0.5851} \\
 & AC $\uparrow$ & 0.4915 & 0.5000 & 0.4787 & 0.5139 & \textbf{0.5316} & 0.5111 & 0.5251 \\
 & MM $\uparrow$ & 0.5907 & 0.6133 & 0.6037 & 0.6238 & 0.6215 & 0.6331 & \textbf{0.6384} \\
 & Mean $\uparrow$ & 0.5762 & 0.5962 & 0.5870 & 0.6053 & 0.6116 & 0.6005 & \textbf{0.6143} \\
 & WT2 PPL $\downarrow$ & 16.3457 & 14.9685 & 15.4972 & 15.3845 & 15.1600 & 15.2118 & \textbf{14.8861} \\
\bottomrule
\end{tabular}}
\end{table*}

\vparagraph{INT4.}
Under dynamic INT4, FlatQuant is the strongest baseline at $0.5461$ five-task mean accuracy, while LoopQ reaches $0.5842$. LoopQ improves all five task accuracies over FlatQuant and reduces WikiText-2 word PPL from $18.84$ to $16.63$. These gains indicate that online range estimation alone does not eliminate the benefit of loop-aware adaptation. The quantization granularities described in Appendix~\ref{apx:implementation_details} also differ from the static setting, so comparisons with static results do not isolate the effect of online range estimation. Runtime costs are reported in Appendix~\ref{apx:deployment_efficiency}.

\vparagraph{NVFP4.}
LoopQ extends effectively to NVFP4, a recent low-precision floating-point format. With a per-loop global scale, LoopQ attains the highest five-task mean accuracy of $0.6143$ and the lowest WikiText-2 word PPL of $14.8861$. Its mean accuracy exceeds MixQuant's $0.6116$ by $0.28$ percentage points and FlatQuant's $0.6005$ by $1.38$ points. LoopQ leads on HellaSwag, LAMBADA, and MMLU, while MixQuant leads on WinoGrande and ARC-C. These results show that loop-aware adaptation remains beneficial with the newer floating-point format and online activation scaling.

\subsection{Early Exit and Loop-Depth Extrapolation}
\label{apx:loop_computation}

We evaluate early exit on Ouro 1.4B and execution beyond the calibrated loop depth on Parcae. For early exit, we sweep $q\in\{0.125,0.250,0.375,0.500,0.625,0.750,0.875\}$ with a maximum of four loops, using the W4A4KV16 checkpoint from Table~\ref{tab:main_results}. Figure~\ref{fig:early_exit}(a,b) reports five-task mean accuracy and average executed iterations. Without early exit, LoopQ achieves $0.5909$ accuracy at four iterations.

Following Ouro's deterministic Q-exit criterion~\citep{zhu2025scaling}, let $\lambda_t(x)$ denote the learned halting probability at loop $t$. The cumulative exit probability is
\begin{equation}
\operatorname{CDF}(t\mid x)
=\sum_{i=1}^{t}\lambda_i(x)\prod_{j=1}^{i-1}(1-\lambda_j(x)).
\end{equation}
Inference stops at the first loop with $\operatorname{CDF}(t\mid x)\ge q$, or after four loops if the threshold is not reached. Smaller thresholds permit earlier exits.

\begin{figure}[t]
\centering
\includegraphics[width=\textwidth]{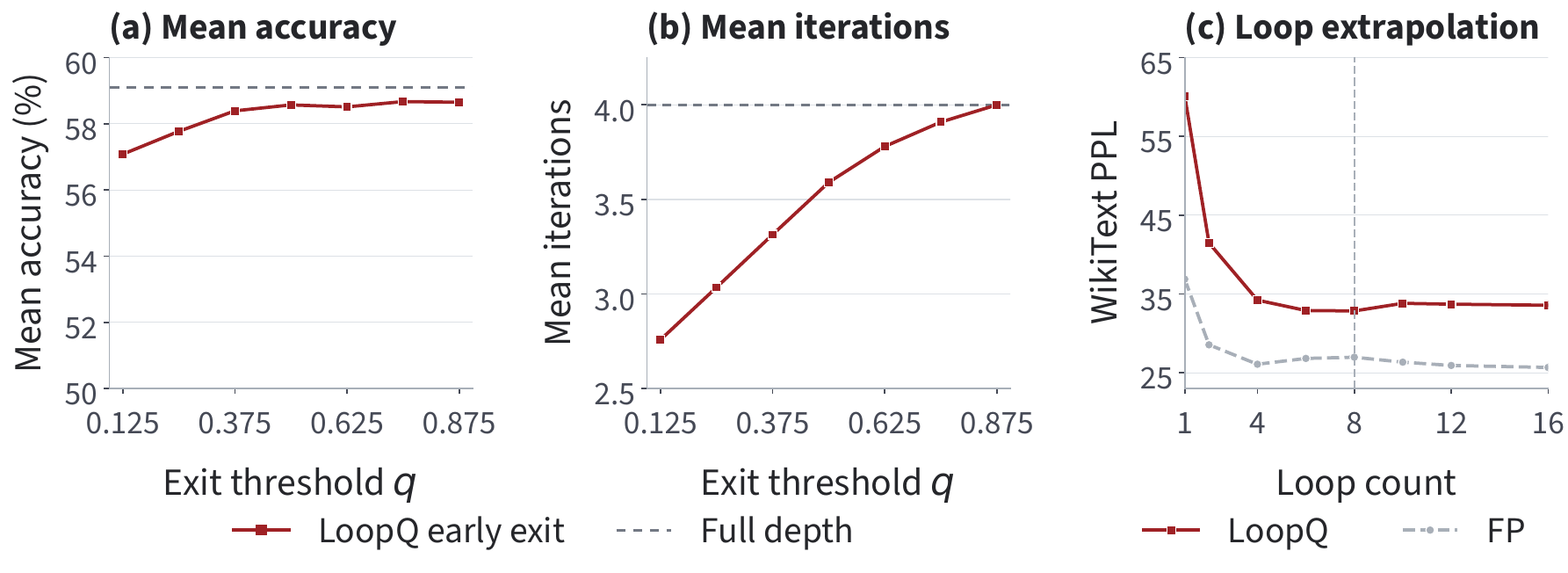}
\caption{Early exit and loop-depth extrapolation with LoopQ. On Ouro 1.4B, (a) five-task mean accuracy and (b) average executed iterations vary with the exit threshold. (c) WikiText-2 PPL on Parcae as loop depth increases; the vertical dashed line marks the eight-loop calibration horizon.}
\label{fig:early_exit}
\end{figure}

LoopQ retains a controllable accuracy--computation trade-off. At $q=0.250$, it averages $3.035$ iterations, $24.1\%$ fewer than full depth, with an accuracy decrease of $1.32$ percentage points. At $q=0.500$, it averages $3.590$ iterations, a $10.3\%$ reduction, with a decrease of $0.52$ percentage points. Higher thresholds approach four iterations as accuracy plateaus.

For loop-depth extrapolation, we calibrate Parcae over eight loops and evaluate it under W4A4 at 10, 12, and 16 loops. Each additional loop uses the geometric mean of activation scales from the last four calibrated loops. LoopQ's WikiText-2 PPL decreases through eight loops, then remains between $33.57$ and $33.82$ at the evaluated extrapolation depths, compared with $32.85$ at eight loops. The full-precision reference ranges from $25.68$ to $26.35$ at the extrapolation depths, compared with $26.98$ at eight loops. Thus, LoopQ supports early exit with a measurable accuracy--computation trade-off and maintains stable perplexity beyond its calibrated depth.

\subsection{Memory Footprint and Inference Throughput}
\label{apx:deployment_efficiency}

\vparagraph{Model-State Accounting.}
\label{apx:storage_accuracy}
Table~\ref{tab:overhead_memory} reports parameter overhead and theoretical deployment memory under the static W4A4 setting in Appendix~\ref{apx:implementation_details}. Shared and loop-dependent metadata are measured relative to total reconstructible values, counting backbone weights once and excluding derived weight copies. Deployment memory instead includes expanded loop-specific weight copies, INT4 codes, FP32 quantization scales, retained state, and embedding codes and scales; it excludes KV caches and temporary workspaces. LoopFormer and Parcae use per-row INT8 embeddings and tied output heads with FP32 scales, while Ouro retains BF16 embeddings. Across the three backbones, loop-dependent metadata accounts for $0.60\%$--$1.14\%$ of total values, and LoopQ reduces theoretical deployment memory by $36.01\%$--$57.68\%$ relative to BF16.

\begin{table}[t]
\centering\small
\caption{Parameter overhead and theoretical deployment memory under static W4A4.}
\label{tab:overhead_memory}
\setlength{\tabcolsep}{8pt}
\renewcommand{\arraystretch}{1.08}
\begin{tabular*}{\linewidth}{@{\extracolsep{\fill}}lrrr@{}}
\toprule
\textbf{Metric} & \textbf{Ouro 1.4B} & \textbf{LoopFormer 3$\times$8} & \textbf{Parcae 370M} \\
\midrule
\multicolumn{4}{@{}l}{\textit{Parameter counts and metadata}} \\
\addlinespace[2pt]
Total values & 1.48B & 0.29B & 0.40B \\
Shared metadata & 41.13M & 3.86M & 5.92M \\
Shared ratio & 2.77\% & 1.36\% & 1.49\% \\
Loop-dependent metadata & 8.93M & 3.25M & 2.83M \\
Loop-dependent ratio & 0.60\% & 1.14\% & 0.71\% \\
\midrule
\multicolumn{4}{@{}l}{\textit{Deployment memory}} \\
BF16 & 2.67~GiB & 0.52~GiB & 0.72~GiB \\
LoopQ W4A4 & 1.13~GiB & 0.34~GiB & 0.34~GiB \\
Relative reduction & $-57.68\%$ & $-36.01\%$ & $-52.99\%$ \\
\bottomrule
\end{tabular*}

\end{table}

\vparagraph{Accuracy--Memory Comparison.}
Table~\ref{tab:deployment_accuracy} and Figure~\ref{fig:dense_memory_accuracy} compare Ouro 1.4B with Qwen3 dense models using five-task mean accuracy. The unrolled Ouro reference applies FlatQuant to separate copies of the loop instances. The Static SOTA point combines the best task accuracies among the static W4A4 baselines for Ouro 1.4B in Table~\ref{tab:main_results} and uses FlatQuant's deployment memory as its representative footprint; it is a metric-wise reference rather than one deployable checkpoint. Quantized Qwen3 checkpoints use NVFP4 weights and activations. Their memory is measured after loading, whereas Ouro memory is estimated from deployment state.

LoopQ exceeds quantized Qwen3-1.7B by $9.82$ percentage points at a smaller reported footprint ($1.13$ versus $1.92$~GiB). Relative to quantized Qwen3-4B, it is $1.04$ percentage points lower while using approximately one-third of the reported memory. Unrolled Ouro gains $2.47$ percentage points over LoopQ but requires an estimated $2.83$~GiB, slightly more than Ouro BF16.

\begin{figure}[t]
\centering
\includegraphics[width=\linewidth]{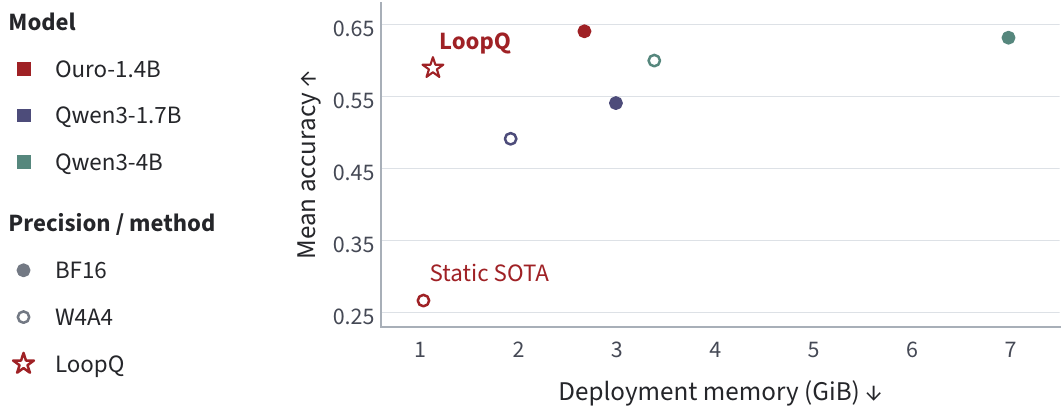}
\caption{Five-task mean accuracy versus memory footprint across models and precisions.}
\label{fig:dense_memory_accuracy}
\end{figure}

\begin{table}[t]
\centering
\footnotesize
\caption{Decode throughput and component time shares for packed W4A4KV4 on Ouro 1.4B-Thinking with RTX 5090.}
\label{tab:packed_runtime}
\setlength{\tabcolsep}{2pt}
\renewcommand{\arraystretch}{1.15}
\begin{minipage}[t]{0.43\linewidth}
\centering
\textbf{(a) Decode throughput (tokens/s)}\par\smallskip
\begin{tabular*}{\linewidth}{@{\extracolsep{\fill}}rrrr@{}}
\toprule
\multirow{2}{*}{\shortstack{Input\\tokens}} & \multicolumn{2}{c}{Throughput} & \multirow{2}{*}{\shortstack{Decode\\speedup}} \\
\cmidrule(lr){2-3}
 & LoopQ & BF16 & \\
\midrule
512   & 1,040.1 & 836.1 & $1.24\times$ \\
1,024 & 844.5 & 611.2 & $1.38\times$ \\
2,048 & 567.4 & 387.3 & $1.47\times$ \\
4,096 & 337.9 & 205.6 & $1.64\times$ \\
8,192 & 180.1 & 84.7 & $2.13\times$ \\
\bottomrule
\end{tabular*}
\end{minipage}\hfill
\begin{minipage}[t]{0.55\linewidth}
\centering
\textbf{(b) GPU execution time shares}\par\smallskip
\begin{tabular*}{\linewidth}{@{\extracolsep{\fill}}rrrrr@{}}
\toprule
\multirow{2}{*}{\shortstack{Input\\tokens}} & \multicolumn{2}{c}{Prefill} & \multicolumn{2}{c}{Decode} \\
\cmidrule(lr){2-3}\cmidrule(lr){4-5}
 & Transforms & CTA & Transforms & CTA \\
\midrule
512   & 11.57\% & 0.46\% & 13.46\% & 0.53\% \\
1,024 & 11.45\% & 0.45\% & 12.87\% & 0.53\% \\
2,048 & 11.05\% & 0.43\% & 12.42\% & 0.54\% \\
4,096 & 10.73\% & 0.42\% & 11.85\% & 0.71\% \\
8,192 & 9.65\% & 0.40\% & 11.74\% & 0.70\% \\
\bottomrule
\end{tabular*}
\end{minipage}
\end{table}

\vparagraph{Packed Inference Protocol.}
\label{apx:packed_runtime}
We evaluate Ouro 1.4B-Thinking on an NVIDIA GeForce RTX 5090 using packed W4A4KV4 with static activation scales and FP32 KV scales and offsets. The BF16 baseline combines vLLM scheduling, FlashAttention 2, and full decode CUDA Graph replay. Both runtimes use vLLM 0.27.1/V1; LoopQ uses custom attention. Inductor compilation, prefix caching, chunked prefill, and asynchronous scheduling are disabled, with a maximum sequence length and prefill token budget of $9{,}216$.

Prompts contain 512--8,192 tokens, and each request generates exactly $1{,}024$ tokens with greedy decoding and no EOS stopping. All requests are enqueued before scheduling. Decode throughput is measured from the point at which every request has produced its first token until all requests finish, including CPU scheduling and the completion tail. Measurements use a local offline engine without HTTP overhead or KV-cache preemption. All BF16 and W4A4 results are medians of five runs. Concurrency is selected separately for each runtime, so speedups represent achieved throughput rather than matched-concurrency execution.

\vparagraph{Throughput and Component Overhead.}
Table~\ref{tab:packed_runtime} and Figure~\ref{fig:decode_speedup} show decode speedups of $1.24$--$2.13\times$ over BF16 across the evaluated prompt lengths. Transformations account for $9.65$--$11.57\%$ of prefill and $11.74$--$13.46\%$ of decode GPU execution time, while CTA accounts for less than $1\%$ in both phases. Dynamic PTG performs online range reduction before A4 packing, whereas LAS retrieves an offline-calibrated scale for the current loop. In a separate RTX 5090 packed-runtime comparison, replacing LAS with Dynamic PTG increases prefill latency by $30.02\%$ and decode latency by $2.15\%$.

\end{document}